\documentclass{article}

\usepackage{microtype}
\usepackage{graphicx}
\usepackage{subfigure}
\usepackage{booktabs} % for professional tables
\usepackage{enumitem}

\usepackage{hyperref}

\usepackage[accepted]{icml2020}

\usepackage{microtype}
\usepackage{cancel}
\usepackage{graphicx}
\usepackage{booktabs} % for professional tables
\usepackage{amssymb}
\usepackage{color}
\usepackage{amsmath}
\usepackage{amsthm}
\usepackage{thmtools}
\usepackage{thm-restate}
\usepackage{algorithm}
\usepackage{algorithmic}
\usepackage{color}

\usepackage{hyperref}
\allowdisplaybreaks

\icmltitlerunning{Improved Sleeping Bandits with Stochastic Actions Sets and Adversarial Rewards}

\newtheorem{rem}{Remark}

\renewcommand{\P}{{\mathbf P}}

\newcommand{\E}{{\mathbf E}}

\newcommand{\1}{{\mathbf 1}}

\newcommand{\cA}{{\mathcal A}}

\newcommand{\cH}{{\mathcal H}}

\newcommand{\cE}{{\mathcal E}}

\newcommand{\cP}{{\mathcal P}}

\renewcommand{\a}{{\mathbf a}}

\newcommand{\p}{{\mathbf p}}
\newcommand{\q}{{\mathbf q}}

\newcommand{\sm}{\setminus}

\newcommand{\bell}{\boldsymbol \ell}

\newcommand{\bmu}{{\boldsymbol \mu}}

\def \algicat{{\it Sleeping-EXP3}}
\def \algicatg{{\it Sleeping-EXP3G}}
\def \algscat{{\it Sleeping-Cat}}
\def \algicatp{{\it Sleeping-EXP3$+$}}
\def \algkand{{\it Bandit-SFPL}}

\newcommand{\red}[1]{\textcolor{red}{#1}}

\renewcommand{\hat}{\widehat}

\begin{document}

\twocolumn[
\icmltitle{Improved Sleeping Bandits with Stochastic Actions Sets \\and Adversarial Rewards}
%




\begin{icmlauthorlist}
\icmlauthor{Aadirupa Saha}{iisc}
\icmlauthor{Pierre Gaillard}{inria}
\icmlauthor{Michal Valko}{deep}
\end{icmlauthorlist}

\icmlaffiliation{iisc}{Indian Institute of Science, Bangalore, India.}
\icmlaffiliation{inria}{Sierra Team, Inria, Paris, France.}
\icmlaffiliation{deep}{DeepMind, Paris, France}

\icmlcorrespondingauthor{Aadirupa Saha}{aadirupa@iisc.ac.in}

\icmlkeywords{Machine Learning, ICML}

\vskip 0.3in
]



\printAffiliationsAndNotice{}  % leave blank if no need to mention equal contribution

\begin{abstract}
In this paper, we consider the problem of sleeping bandits with stochastic action sets and adversarial rewards. In this setting, in contrast to most work in bandits, the actions may not be available at all times. For instance, some products might be out of stock in item recommendation. The best existing efficient (i.e., polynomial-time) algorithms for this problem only guarantee an $O(T^{2/3})$ upper-bound on the regret. Yet, inefficient algorithms based on EXP4 can achieve $O(\sqrt{T})$. In this paper, we provide a new computationally efficient algorithm inspired by EXP3 satisfying a regret of order $O(\sqrt{T})$ when the availabilities of each action $i \in \cA$ are independent. We then study the most general version of the problem where at each round available sets are generated from some unknown arbitrary distribution (i.e., without the independence assumption) and propose an efficient algorithm with $O(\sqrt {2^K T})$ regret guarantee. Our theoretical results are corroborated with experimental evaluations.
\iffalse %%%%%%%%%%%%%%%%%%%%%%%%%%%%%%%
\red{To be updated:} In this paper we aim to give computationally efficient and optimal algorithms for the problem of sleeping bandits with stochastic reward and adversarial availabilities. Towards this we propose a near optimal $O(K^{3/2}\sqrt T)$ algorithm for the setup when the availabilities of each item $i \in \cA$, are independent. We derive both problem dependent and independent bounds: While the former specifically depends on the availability probability of each individual items, the later bound yields a regret guarantee for case instances. We then study the most general version of the problem where at each round available sets are generated from some unknown distribution and proposed an algorithm with $O(\sqrt {2^K T})$ regret guarantee, which to the best of our knowledge is the first algorithm to attain an $O(\sqrt T)$ guarantee for this setup. Our theoretical results are corroborated with experimental evaluations.
\fi %%%%%%%%%%%%%%%%%%%%%%%%%%%%%%%%%%%%%
\end{abstract}

%!TEX root = sleeping_icml20.tex
\section{Introduction}
\label{intro}

The problem of standard multiarmed bandit (MAB) is well studied in machine learning \cite{Auer00,mohri05} and used to model online decision-making problems under uncertainty. 
Due to their implicit exploration-vs-exploitation tradeoff, bandits are able to model clinical trials, movie recommendations, retail management job scheduling etc., where the goal is to keep pulling the `best-item' in hindsight through sequentially querying one item at a time and subsequently observing a noisy reward feedback of the queried arm \cite{Even+06,Auer+02,Auer02,TS12,BubeckNotes+12}.
However, in various real world applications, the decision space (set of arms $\cA$) often changes over time due to unavailability of some items etc.  For instance, in retail stores some items might go out of stock, on a certain day some websites could be down, some restaurants might be closed etc. This setting is known as \emph{sleeping bandits} in online learning  \cite{kanade09,neu14,kanade14,kale16}, where at any round the set of available actions could vary stochastically based on some unknown distributions over $\cA$ \cite{neu14,cortes+19} or adversarially \cite{kale16,kleinberg+10,kanade14}. Besides the reward model, the set of available actions could also vary stochastically or adversarially \cite{kanade09,neu14}. The problem is known to be NP-hard when both rewards and availabilities are adversarial \cite{kleinberg+10,kanade14,kale16}.
In case of stochastic rewards and adversarial availabilities the achievable regret lower bound is known to be $\Omega(\sqrt{KT})$, $K$ being the number of actions in the decision space $\cA = [K]$. The well studied EXP$4$ algorithm does achieve the above optimal regret bound, although it is  computationally inefficient \cite{kleinberg+10,kale16}. However, the best known efficient algorithm only guarantees an $\tilde O((TK)^{2/3})$ regret,\footnote{$\tilde O(\cdot)$ notation hides the logarithmic dependencies.} which is not matching the lower bound both in $K$ and $T$ \cite{neu14}.

In this paper we aim to give computationally efficient and optimal $O(\sqrt{T})$ algorithms for the problem of sleeping bandits with adversarial rewards and stochastic  availabilities. Our specific contributions are as follows:

\textbf{Contributions}  
\begin{itemize}[itemsep=2pt,parsep=2pt,topsep=0pt]

\item We identified a drawback in the (sleeping) loss estimates in the prior work for this setting and gave an insight and margin for improvement 
over the best know rate (Section~\ref{subsec:alg_inst}).

\item We first study the setting when the availabilities of each item $i \in \cA$ are independent and propose an EXP$3$-based algorithm (Alg.\,\ref{alg:icat}) with an $O(K^{2}\sqrt T)$ regret guarantee (Theorem\,\ref{thm:reg_icat}, Sec.\,\ref{sec:algo_indep}).
	
\item We next study the problem when availabilities are not independent and give an algorithm with an $O(\sqrt {2^K T})$ regret guarantee (Sec.\,\ref{sec:algo_dep}).
	
\item We  corroborated our theoretical results with empirical evidence (Sec.\,\ref{sec:expts}).
\end{itemize}

%\textbf{Organization.} We introduce the formal problem statement in Sec.\,\ref{sec:prob}. Sec.\,\ref{sec:algo_indep} presents our results when availabilities of each item are independent of each other. A more general regret analysis is provided in Sec.~\ref{sec:algo_indep}. Our experimental evaluations are given in Sec.\,\ref{sec:expts}. We finally conclude the paper in Sec.\,\ref{sec:concl} with some future scopes.

%!TEX root = sleeping_icml20.tex
\section{Problem Statement}
\label{sec:prob}

%\input{prelims.tex}

%\textbf{Preliminaries}
\textbf{Notation.} 
\label{prelims}
We denote by $[n]: = \{1,2,\ldots, n\}$. 
$\1(\cdot)$ denotes the indicator random variable which takes value $1$ if the predicate is true and $0$ otherwise.
$\tilde O(\cdot)$ notation is used to hide logarithmic dependencies.
%We used the same problem setup of sleeping bandits introduced by \cite{neu14,kanade09}. 

\subsection{Setup}
\label{sec:setup}
Suppose the decision space (or set of actions) is $[K]: = \{1,2,\ldots, K\}$ with $K$ distinct actions, and we consider a $T$ round sequential game.
At each time step $t \in [T]$, the learner is presented a set of available actions at round $t$, say $S_t \subseteq [K]$, upon which the learner's task is to play an action $i_t \in S_t$ and consequently suffer a loss $\ell_t(i_t) \in [0,1]$, where $\bell_t:= [\ell_t(i)]_{i \in [K]} \in [0,1]^K$ denotes the loss of the $K$ actions chosen obliviously independent of the available actions $S_t$ at time $t$. We consider the following two types of availabilities:

\textbf{Independent Availabilities. } In this case we assume that the availability of each item $i \in [K]$ is independent of the rest $[K]\sm\{i\}$, such that at each round item $i \in [K]$ is drawn in set $S_t$ with probability $a_i \in [0,1]$, or in other words, for all item $i \in [K]$, $\1(i \in S_t) \sim Ber(a_i)$, where availability probabilities $\{a_i\}_{i \in [K]}$ are fixed over time intervals $t$, independent of each other, and unknown to the learner.

\textbf{General Availabilities. } In this case each $S_t$s is drawn iid from some unknown distribution $\cP$ over  subsets $\{S \subseteq [K], |S| \ge 1 \}$
with no further assumption made on the properties of $\cP$. We denote by $P(S)$ the probability of occurrence of set $S$.

Analyses with independent and general availabilities are provided respectively in Sec.~\ref{sec:algo_indep} and~\ref{sec:algo_dep}.

\subsection{Objective}
\label{sec:obj}
We define by a policy $\pi: 2^{[K]} \mapsto [K]$ to be a mapping from a set of available actions/experts to an item. 

\textbf{Regret definition}
The performance of the learner, measured with respect to the best policy in hindsight, is defined as:
\begin{align}
\label{eq:reg}
R_{T} = \max_{\pi: 2^{[K]} \mapsto [K]} \E\bigg[ \sum_{t=1}^{T}\ell_t(i_t) - \sum_{t=1}^{T}\ell_t(\pi(S_t)) \bigg],
\end{align}
where the expectation is taken w.r.t. the availabilities and the randomness of the player's strategy.

\begin{rem}
\label{rem:lb}
One obvious regret lower bound of the above objective is $\Omega(\sqrt {KT})$, 
which follows from the bound of standard MAB with adversarial losses \cite{Auer+02} for the special case when all the items are available at all times (even for availability-independent case). Interestingly, for a \emph{harder} Sleeping-Bandits setting with adversarial availabilities the lower bound is $\Omega(K\sqrt T)$ \cite{kleinberg+10}, even for which no computationally efficient algorithm is known till date (EXP4 is the only algorithm which achieves the regret but it is computationally inefficient). Thus the interesting question to answer here is if for our setup--that lies in the middle-ground of \cite{Auer+02} and \cite{kleinberg+10}--is it possible to attend the $O(K \sqrt{T})$ learning rate? Here lies the primary objective of this work. To the best of our knowledge, there is no existing algorithm which are known to achieve this optimal rate and the best known efficient algorithm is only guaranteed to yield an $\tilde O((TK)^{2/3})$ regret \cite{neu14}. 
%we believe this to be the case which could possibly be derived by constructing suitably hard problem instances. 
\end{rem}

%!TEX root = sleeping_icml20.tex
\section{Proposed algorithm: Independent Availabilities}
\label{sec:algo_indep}

In this section we propose our first algorithm for the problem (Sec.~\ref{sec:prob}), which is based on a variant of thr EXP3 algorithm with a `suitable' loss estimation technique. Thm.~\ref{thm:reg_icat} proves the optimality of its regret performance.

\label{subsec:alg_inst}
\textbf{Algorithm description. } Similar to EXP$3$  algorithm, at every round $t\in [T]$ we maintain a probability distribution $\p_t$ over the arm set $[K]$ and also the empirical availability of each item 
$
	\hat a_{ti} = \frac{1}{t} \sum_{\tau = 1}^{t}\1(i \in S_\tau) \,.
$ 
Upon receiving the available set $S_t$, the algorithm redistributes $\p_t$ only on the set of available items $S_t$, say $\q_t$, and plays an item $i_t \sim \q_t$. Subsequently the environment reveals the loss $\ell_t(i_t)$, and we update the distribution $\p_{t+1}$ using exponential weights on the loss estimates for all $i \in [K]$
\begin{equation}
	\label{eq:loss_estimate}
	\hat \ell_t(i) = \frac{\ell_t(i)\1(i = i_t)}{\bar q_t(i) + \lambda_t} \,,
\end{equation}
where $\lambda_t$ is a scale parameter and $\bar q_t(i)$ (see definition~\eqref{eq:barqt_def}) is an estimation of $Pr_{S_t, i_t}\big(i_t = i\big)$, the probability of playing arm $i$ at time $t$ under the joint uncertainty in availability of~$i$ (due to $S_t \sim \cP$) and the randomness of EXP3 algorithm (due to $i_t \sim \p_t$).

\emph{New insight compared to existing algorithms.} It is crucial to note that one of our main contributions lies in the loss estimation technique $\hat \bell_t$ in~\eqref{eq:loss_estimate}. The standard loss estimates used by EXP3 (see \cite{auer02nonstochastic}) are of the form $\hat \ell_t(i) =  \ell_t(i) \1(i = i_t)/p_t(i)$. Yet, because of the unavailable actions, the latter is biased. The solution proposed by \cite{neu14} (see Sec. 4.3) consists of using unbiased loss estimates of the form $\hat \ell_t(i) =  \ell_t(i) \1(i = i_t)/ (\hat p_t(i) \hat a_{ti})$ where $\hat{a}_{ti}$ and $\hat p_t(i)$ are estimates for the availability probability $a_i$ and for the weight $p_t(i)$ respectively. The suboptimal $O(T^{2/3})$ of their regret bound  resulted from this separated estimation of $\hat p_t(i)$ and $\hat a_{ti}$, which leads to a high variance in the analysis because $\hat p_t(i) = 0$ whenever $i \notin S_t$. 

We circumvent this problem by estimating them jointly as 
\begin{equation}
	\bar q_t(i):= \sum_{S \in 2^{[K]}}P_{\hat \a}(S)q_{t}^{S}(i) \,,
	\label{eq:barqt_def}
\end{equation}
where $P_{\hat \a_t}(S) = \Pi_{i =1}^{K}\hat a_{ti}^{\1(i \in S)} (1-\hat a_{ti})^{1-\1(i \in S)}$ is the empirical probability of the availability of set $S$, and for all $i \in [K]$
\begin{equation}
	\label{eq:qtS_def}
	q_{t}^{S}(i) := \frac{p_t(i)\1(i \in S)}{\sum_{j \in S}p_t(j)},\, 
\end{equation}
is the redistributed mass of $\p_t$ on support set $S$. As shown in Lem.~\ref{lem:conc_barq_p}, $\bar q_t(i)$ is a good estimate for $q_t^*(i) = \E_{S \sim \a}\big[ q_{t}^{S}(i) \big]$, which is the conditional probability of playing action $i_t = i$ at time $t$. It turns out that $\bar q_t(i)$ is much more stable than $\hat p_t(i) \hat a_{ti}$ and therefore implies better variance control in the regret analysis. This improvement finally leads to the optimal $O(\sqrt T)$ regret guarantee (Thm.~\ref{thm:reg_icat}). The complete algorithm is given in Alg.~\ref{alg:icat}.

%\vspace*{-25pt}
\begin{center}
\begin{algorithm}[t]
   \caption{\textbf{\algicat} }
   \label{alg:icat}
\begin{algorithmic}[1]
   \STATE {\bfseries Input:} 
   \STATE ~~~ Item set: $[K]$, learning rate $\eta$, scale parameter $\lambda_t$
   \STATE ~~~ Confidence parameter: $\delta >0$
   \STATE {\bfseries Initialize:} 
   \STATE ~~~ Initial probability distribution $\p_1(i) = \frac{1}{K}, ~\forall i \in [K]$
   \WHILE {$t = 1, 2, \ldots$}
   \STATE Define $q_t^S(i) := \frac{p_t(i)\1(i \in S)}{\sum_{j \in S}p_t(j)},\, \forall i \in [K], S\subseteq [K]$
   \STATE Receive $S_t \subseteq [K]$
   \STATE Sample $i_t \sim \q_t^{S_t}$
   \STATE Receive loss $\ell_t(i_t)$
   \STATE Compute: $\hat a_{ti} = \frac{\sum_{\tau = 1}^{t}\1(i \in S_\tau)}{t}$
   \STATE \hspace*{40pt} $P_{\hat \a}(S) = \Pi_{i =1}^{K}\hat a_{ti}^{\1(i \in S)} (1-\hat a_{ti})^{1-\1(i \in S)}$
   \STATE \hspace*{40pt} $\bar q_t(i) = \sum_{S \in 2^{[K]}}P_{\hat \a}(S)q_{t}^{S}(i)$
   \STATE Estimate loss: $\hat \ell_t(i) = \frac{\ell_t(i)\1(i = i_t)}{\bar q_t(i) + \lambda_t}$, $\forall i \in [K]$
   %$p'_{t+1}(i) = p_{t}(i)e^{-\eta \hat \ell_t(i)}$
   \STATE Update $p_{t+1}(i) = \frac{p_{t}(i)e^{-\eta \hat \ell_t(i)}}{\sum_{j = 1}^{K}p_{t}(i)e^{-\eta \hat \ell_t(j)}},\, \forall i \in [K]$
   \ENDWHILE
   %\STATE {\bfseries Output:} 
\end{algorithmic}
\end{algorithm}
\vspace{-2pt}
\end{center}
%\vspace{-15pt}

The first crucial result we derive towards proving Thm.~\ref{thm:reg_icat} is the following concentration guarantees on $\bar q_t$:

\begin{restatable}[Concentration of $\bar \q_t$]{lem}{concbarqp}
	\label{lem:conc_barq_p}
	Let $t\in [T]$ and $\delta \in (0,1)$. Let $q^*_t(i) = \E_{S \sim \a}\big[ q_{t}^{S}(i) \big]$ and $\bar q_t$ as defined in Equation~\eqref{eq:barqt_def}. Then, with probability at least $1-\delta$,  
	\begin{equation}
	\label{eq:barq_conc}
	|q^*_t(i) - \bar q_t(i)| \le 2 K \sqrt{\frac{2\log(K/\delta)}{t}} + \frac{8K\log(K/\delta)}{3t} \,,
	\end{equation}
	for all $i \in [K]$.
\end{restatable}

Using the result of Lem.~\ref{lem:conc_barq_p}, the following theorem analyses the regret guarantee of \algicat \, (Alg.~\ref{alg:icat}).

%\subsection{Regret Analysis}
%\label{sec:reg_bnd}

\begin{restatable}[\algicat: Regret Analysis]{thm}{ubicat}
\label{thm:reg_icat}
Let $T \geq 1$. The sleeping regret incurred by \algicat\, (Alg.~\ref{alg:icat}) can be bounded as:
\begin{eqnarray*}
R_T & = & \max_{\pi: 2^{[K]} \mapsto [K]}\E\bigg[ \sum_{t=1}^{T}\ell(i_t) - \sum_{t=1}^{T}\ell(\pi(S_t)) \bigg] \\
& \le & 16 K^{2} \sqrt{ T \ln T}  + 1 \,,
\end{eqnarray*}
for the parameter choices $\eta  =  \sqrt{(\log K)/(KT)}$,  $\delta  =  K/T^2$, and 
\[
\lambda_t  = \min\left\{  2 K \sqrt{\frac{2\log(K/\delta)}{t}} + \frac{8K\log(K/\delta)}{3t},1\right\} \,.
\]
\end{restatable}

\begin{proof}\textbf{\hspace{-3pt}(sketch) }
Our proof is developed based on the standard regret guarantee of the EXP3 algorithm for the classical problem of multiarmed bandits with adversarial losses \cite{auer02nonstochastic,Auer02}. 
Precisely, consider any fixed set $S \subseteq [K]$, and suppose we run EXP3 algorithm on the set $S$, over any nonnegative sequence of losses $\hat \ell_1, \hat \ell_2, \ldots \hat \ell_T$ over items of set $S$, and consequently with weight updates $\q_1^S, \q_2^S, \ldots \q_T^S$ as per the EXP3 algorithm with learning rate $\eta$. Then from the standard regret analysis of the EXP3 algorithm \cite{PLG06}, we get that for any $i \in S$:
\begin{align*}
\sum_{t = 1}^{T}\big< \q_t^S,\hat \ell_t \big> - \sum_{t = 1}^T\hat \ell_t(i)  \le \frac{\log K}{\eta} + \eta \sum_{t = 1}^T \sum_{k \in S} q_t^S(k)\hat \ell_t(k)^2 \,.
\end{align*}

Let $\pi^*:S \mapsto [K]$ be any strategy. Then, applying the above regret bound to the choice $i = \pi^*(S)$ and taking the expectation over $S \sim P_{\a}$ and over the possible randomness of the estimated losses, we get
\begin{align}
\sum_{t = 1}^{T} \E\Big[\big< \q_t^S,& \hat \ell_t \big>\Big] - \sum_{t = 1}^T\E\Big[\hat \ell_t(\pi^*(S))\Big] \le \nonumber \\
& \frac{\log K}{\eta} + \eta \sum_{t = 1}^T \E\bigg[\sum_{k \in S} q_t^S(k) \hat \ell_t(k)^2 \bigg] \,.
\label{eq:RegretEXP31}
\end{align}

Now towards proving the actual regret bound of \algicat\, (recall the definition from Eqn. \eqref{eq:reg}), we first need to establish the following three main sub-results that relate the different expectations of Inequality~\eqref{eq:RegretEXP31} with quantities related to the actual regret (in Eqn. \eqref{eq:reg}). 

\begin{restatable}[]{lem}{firstterm}
	\label{lem:first_term}  
	Let $\delta \in (0,1)$. Let $t \in [T]$. Define $\q_t^S$ as in~\eqref{eq:qtS_def} and $\hat \ell_t$ as in~\eqref{eq:loss_estimate}. Assume that $i_t$ is drawn according to $\q_t^{S_t}$ as defined in Alg.~\ref{alg:icat}. Then,
	\[
	\E\big[\ell_t(i_t)\big] \le 
	\E\big[\big< \q_t^S,\hat \ell_t \big>\big] + 2 K \lambda_t + \frac{\delta}{\lambda_t} \,,
	\]
	for $\lambda_t = 2 K \sqrt{\frac{2\log(K/\delta)}{t}} + \frac{8K\log(K/\delta)}{3t}$. 
\end{restatable}

\begin{restatable}[]{lem}{secondterm}
	\label{lem:second_term}
	Let $\delta \in (0,1)$. Let $t \in [T]$. Define $\hat \ell_t$ as in~\eqref{eq:loss_estimate} and assume that $i_t$ is drawn according to $\q_t^{S_t}$ as defined in Alg.~\ref{alg:icat}. Then for any $i \in [K]$,
	\[
	\E\big[ \hat \ell_t(i)  \big] \le  \ell_t(i) + \frac{\delta}{\lambda_t} \,,
	\]
	for $\lambda_t = 2 K \sqrt{\frac{2\log(K/\delta)}{t}} + \frac{8K\log(K/\delta)}{3t}$.
\end{restatable}

\begin{restatable}[]{lem}{varterm}
	\label{lem:var_term}
	Let $\delta \in (0,1)$. Let $t \in [T]$. Define $\q_t^S$ as in~\eqref{eq:qtS_def} and $\hat \ell_t$ as in~\eqref{eq:loss_estimate}. Then,
	\[
	\E\Bigg[ \sum_{i \in S} q_t^S(i)\hat  \ell_t(i)^2\Bigg] \le K + \frac{\delta}{\lambda_t^2}.
	\]
	for $\lambda_t = 2 K \sqrt{\frac{2\log(K/\delta)}{t}} + \frac{8K\log(K/\delta)}{3t}$. 
\end{restatable}

With the above claims in place, we are now proceed to prove the main theorem: Let us denote the best policy $\pi^* := \arg\min_{\pi: 2^{[K]}\mapsto [K]}\sum_{t = 1}^{T}\E_{S_t \sim P_{\a}}[\ell(\pi(S_t))]$. Now, recalling from Eqn. \eqref{eq:reg}, the actual regret definition of our proposed algorithm, and combining the claims from Lem.~\ref{lem:first_term},~\ref{lem:second_term}, we first get:
\begin{align*}
& R_{T}(\text{\algicat}) =  \sum_{t=1}^{T}\E\Big[\ell_t(i_t) - \ell_t(\pi^*(S_t))\Big]\\
& \le  2K \sum_{t = 1}^{T} \lambda_t + 2 \sum_{t=1}^T \frac{\delta}{\lambda_t} + \sum_{t=1}^{T}\E\Big[\big< \q_t^S,\hat \ell_t \big>  -  \hat \ell_t(\pi^*(S))\Big] \,.
\end{align*}
Then, we can further upper-bound the last term on the right-hand-side using Inequality~\eqref{eq:RegretEXP31} and  Lem.~\ref{lem:var_term}, which yields
\begin{align}
R_{T}&(\text{\algicat}) \nonumber \\
& \leq 2K \sum_{t = 1}^{T} \lambda_t + 2 \sum_{t=1}^T \frac{\delta}{\lambda_t}  + \frac{\log K}{\eta} + \eta KT + \eta  \sum_{t=1}^T \frac{\delta}{\lambda_t^2} \nonumber \\
& \leq \frac{\log K}{\eta} + \eta KT  +  2K \sum_{t = 1}^{T} \lambda_t + 3 \sum_{t=1}^T \frac{\delta}{\lambda_t^2} \,,
\label{eq:sleeping_regret_bound1}
\end{align}
where in the last inequality we used that  $\eta \leq 1$ and $\lambda_t \leq 1$. Otherwise, we can always choose $\min\{1,\lambda_t\}$ instead of $\lambda_t$ in the algorithm and Lem.~\ref{lem:conc_barq_p} would still be satisfied.

The proof is concluded by replacing $\lambda_t  = 2 K \sqrt{\frac{2\log(K/\delta)}{t}} + \frac{8K\log(K/\delta)}{3t}$ and by bounding the two sums as follows:
\begin{align*}
\sum_{t=1}^T \lambda_t  \leq 2K \sqrt{2 \log \Big(\frac{K}{\delta}\Big) T} + \frac{8K}{3} \log \Big(\frac{K}{\delta}\Big) (1+\log T) 
\end{align*}
and using $\lambda_t \geq 2K\sqrt{2 \log(K/\delta)/t}$, we have
\begin{align*}
\sum_{t=1}^T \frac{1}{\lambda_t^2}
& \leq \frac{1}{8 K^2  \log (K/\delta)} \sum_{t=1}^T t \leq \frac{T^2}{8 K^2  \log (K/\delta)} \leq \frac{T^2}{8 K^2}.
\end{align*}
Then, using $\delta := K/T^2$, $\log(K/\delta) = 2 \log(T)$, we can further upper-bound:
$
\sum_{t=1}^T \lambda_t \leq 4K \sqrt{T\log T} + \frac{8K}{3} (1+\log T)(\log T) \leq 7K \sqrt{T\log T}, 
$
and
$
3 \sum_{t=1}^T \frac{\delta}{\lambda_t^2} \leq \frac{3}{8K} \leq 1. 
$
Thus, upper-bounding the two sums into~\eqref{eq:sleeping_regret_bound1}, we get
\[
R_{T}(\text{\algicat}) \leq \frac{\log K}{\eta} + \eta KT + 14 K^2 \sqrt{T \log T} + 1 \,.
\]
Optimizing $\eta = \sqrt{(\log K)/KT}$ and upper-bounding $\sqrt{KT \log K} \leq K^2 \sqrt{T}$, finally concludes the proof.
\end{proof}
\vspace{-10pt}
The above regret bound is of order $\tilde O(K^{2}\sqrt{T})$, which is optimal in $T$, unlike any previous work which could only achieve $\tilde O((KT)^{2/3})$ regret guarantees \cite{neu14} at best. Thus our regret guarantee is only suboptimal in terms of $K$, as the lower bound of this problem is known to be $\Omega(\sqrt {KT})$ \cite{kleinberg+10,kanade09}. However, it should be noted that in our experiments (see Figure~\ref{fig:reg_vs_K}), the dependence of our regret on the number of arms behaves similarly to other algorithms although their theoretical guarantees expect better dependencies on $K$. The sub-optimality could thus be an artifact of our analysis, but despite our efforts, we have not been able to improve it. We think this may come from our proof of Lem.~\ref{lem:conc_barq_p}, in which we see two gross inequalities that may cost us this dependence on $K$.  First, the proof upper-bounds $|q_t^{(S)}(i)| \leq 1$, while in average over $i$ and $S$ the latter is around $1/K$. Yet, dependence problems condemn us to use this worst-case upper-bound. Secondly, the proof uses uniform bounds of $|a_i - \hat a_{ti}|$ over $i=1,\dots,K$ when the estimation errors could offset each other.

Note also that the regret bound in the theorem is worst-case. An interested direction for future work would be to study whether it is possible to derive an instance-dependent bound, based on the $a_i$ instances. Typically, $K$ could be replaced by the expected number of active experts. A first step in this direction would be to start from Inequality~\eqref{eq:instancedependent_bound} in
the proof of Lem.~\ref{lem:conc_barq_p} and try to keep the dependence on the $a_i$ distribution along the proof. 

Finally, note that the algorithm only requires the beforehand knowledge of the horizon $T$ to tune its hyper-parameter. However, the latter assumption can be removed by using standard calibration techniques such as the doubling trick (see~\cite{PLG06}).

 % which we believe could be improved to faster concentration rate of $|q^*_t(i) - \bar q_t(i)| \le O\big(\sqrt{{K}/{t}}\big)$ by directly using the bound of Eqn. \eqref{eq:concP} (see Appendix~\ref{app:reg_icat}) for $|P_{\hat \a_t}(S) - P_{\a}(S)|$ instead of the one used (i.e. Eqn. \eqref{eq:concP}) and then following the rest of the proof analysis of Thm.~\ref{thm:reg_icat}---it is easy to see that this would lead absorb the extra $K$ factor in the concentration bound leading to optimal $O(\sqrt{KT})$ regret bound. The current analysis is however simpler.

\subsection{Efficient \algicat: Improving Computational Complexity}
\label{subsec:alg_eff}
Thm.~\ref{thm:reg_icat} shows the optimality of \algicat \,(Alg.~\ref{alg:icat}), but its one limitation lies in computing the probability estimates 
\[
\bar q_t(i) := \sum_{S \in 2^{[K]}}P_{\hat \a}(S)q_{t}^{S}(i), \forall {i \in [K]},
\]
which requires $O(2^K)$ computational complexity per round. %$t \geq 1$.

In this section we show how to get around with this problem just by approximating $\bar q_t(i)$ by an empirical estimate 
\begin{equation}
	\label{eq:tildeqt_def}
	\tilde q_t(i) := \frac{1}{t}\sum_{\tau = 1}^{t} q_t^{S^{(\tau)}_t}(i) \,,
\end{equation}
where $S_t^{(1)},S_t^{(2)},\ldots S_t^{(t)}$ are $t$ independent draws from the distribution $P_{\hat \a}$, i.e. $P_{\hat \a_t}(S) := \Pi_{i =1}^{K}\hat a_{ti}^{\1(i \in S)} (1-\hat a_{ti})^{1-\1(i \in S)}$ for any $S \subseteq [K]$ (recall the notation from Sec.~\ref{subsec:alg_inst}).
The above trick proves useful with the crucial observation that $q_t^{S^{(1)}_t}(i), q_t^{S^{(2)}_t}(i), \ldots q_t^{S^{(T)}_t}(i)$ are independent of each other (given the past) and that each $q_t^{S_t^{(\tau)}}(i)$ are unbiased estimated of $\bar q_t(i)$. That is, $\E_{\hat \a_t}[q_t^{S^{(\tau)}_t}(i)] = \bar q_t(i),\, \forall i \in [K], \tau \in [t]$. By classical concentration inequalities, this precisely leads to fast concentration of $\tilde q_t(i)$ to $\bar q_t(i)$ which in turn concentrates to $\q^*_t(i)$ (by Lem.~\ref{lem:conc_barq_p}). Combining these results, thus one can obtain the concentration of $\tilde q_t(i)$ to $\q^*_t(i)$ as shown in Lem.~\ref{lem:conc_barq_p_e}. 

\begin{restatable}[Concentration of $\tilde q_t(i)$]{lem}{concbarqe}
	\label{lem:conc_barq_p_e}
	Let $t \in [T]$ and $\delta \in (0,1)$. Let $q^*_t(i) = \E_{S \sim \a}\big[ q_{t}^{S}(i) \big]$ and $\tilde q_t$ as defined in Equation~\eqref{eq:tildeqt_def}. Then, with probability at least $1-\delta$, 
	\begin{equation*}
	%\label{eq:barq_conc}
	|q^*_t(i) - \tilde q_t(i)| \le   4 K \sqrt{\frac{\log(2K/\delta)}{t}} + \frac{8K\log(2K/\delta)}{3t} \,,
	\end{equation*}
	for all $i \in [K]$.
\end{restatable}

\begin{rem}
Note that for estimating $\tilde \q_t$, we can not use the observed sets $S_1, S_2\ldots, S_t$, instead of resampling $S_t^{(1)}, S_t^{(2)}\ldots, S_t^{(t)}$ again--this is because in that case the resulting numbers $q_t^{S_1}(i), q_t^{S_2}(i), \ldots q_t^{S_t}(i)$ would no longer be independent, and hence can not derive the concentration result of Lem.~\ref{lem:conc_barq_p_e} (see proof in Appendix~\ref{app:conc_icate} for details).
\end{rem}

Using the result of Lem.~\ref{lem:conc_barq_p_e}, we now derive the following theorem towards analyzing the regret guarantee of the  computationally efficient version of \algicat.

\begin{restatable}[\algicat \,(Computationally efficient version): Regret Analysis]{thm}{ubicate}
\label{thm:reg_icat_e}
Let $T\geq 1$. The sleeping regret incurred by the efficient approximation of \algicat\, (Alg.~\ref{alg:icat}) can be bounded as:
\begin{eqnarray*}
R_T %& = & \max_{\pi: 2^{[K]} \mapsto [K]}\E\bigg[ \sum_{t=1}^{T}\ell(i_t) - \sum_{t=1}^{T}\ell(\pi(S_t)) \bigg] \\
	& \le &  20 K^2 \sqrt{T \log T} + 1 \,,
\end{eqnarray*}
for the parameter choices $\eta = \sqrt{\frac{\log K}{KT}}$, $\delta = 2K/T^2$ and  
$\lambda_t := 4 K \sqrt{{\log(2K/\delta)}/{t}} + {8K\log(2K/\delta)}/{3t}$.

Furthermore, the per-round time and space complexities of the algorithm are $O(tK)$ and $O(K)$ respectively. 
\end{restatable}

\begin{proof} \textbf{(sketch)}
The regret bound can be proved using similar steps as described for Thm.~\ref{thm:reg_icat}, except now we replace the concentration result of Lem.~\ref{lem:conc_barq_p_e} in place of Lem.~\ref{lem:conc_barq_p}.

\textbf{Computational complexity}: At any round $t\geq 1$, the algorithm requires only an $O(K)$ cost to update $\hat a_{ti}$, $\hat \ell_t(i)$ and $p_{t+1}(i),\, \forall i \in [K]$. Resampling $t$ subsets $S_t^{(\tau)}$ and computing $\{q_t^{S_t^{(\tau)}}(i)\}_{i \in [K]}$ requires another $O(tK)$ cost, resulting in the claimed computational complexity.

\textbf{Spatial complexity:} We only need to keep track of $\hat \a_t \in [0,1]^K$ and $\p_t \in [0,1]^K$ making the total storage complexity just $O(K)$ (noting $\tilde \q_t$ can be computed sequentially).
\end{proof}

% \input{instance_independent_bound.tex}

%!TEX root = sleeping_icml20.tex
\section{Proposed algorithm: General Availabilities}
\label{sec:algo_dep}

\textbf{Setting. } In this section we assume \emph{general subset availabilities} (see Sec.~\ref{prelims}).

\vspace*{-9pt}
\begin{center}
\begin{algorithm}[H]
   \caption{\textbf{\algicatg} }
   \label{alg:icatg}
\begin{algorithmic}[1]
   \STATE {\bfseries Input:} 
   \STATE ~~~ Learning rate $\eta > 0$, scale parameter $\lambda_t$
   \STATE ~~~ Confidence parameter: $\delta >0$
   \STATE {\bfseries Initialize:} 
   \STATE ~~~ Initial probability distribution $\p_1(i) = \frac{1}{K}, ~\forall i \in [K]$
   \WHILE {$t = 1, 2, \ldots$}
   \STATE Receive $S_t$
   \STATE Compute $q_t(i) = \frac{p_t(i)\1(i \in S_t)}{\sum_{j \in S_t}p_t(j)},\, \forall i \in [K]$
   \STATE Sample $i_t \sim \q_t$
   \STATE Receive loss $\ell_t(i_t)$
   \STATE Compute: $\bar q_t (i): = \frac{1}{t} \sum_{\tau = 1}^{t} q_t^{S_\tau}(i)$
   \STATE Estimate loss bound $\hat \ell_t(i) = \frac{\ell_t(i)\1(i = i_t)}{\bar q_t(i) + \lambda_t}$
   \STATE Update $p_{t+1}(i) = \frac{p_{t}(i)e^{-\eta \hat \ell_t(i)}}{\sum_{j = 1}^{K}p_{t}(i)e^{-\eta \hat \ell_t(j)}},\, \forall i \in [K]$
   \ENDWHILE
   %\STATE {\bfseries Output:} 
\end{algorithmic}
\end{algorithm}
\vspace{-2pt}
\end{center}
\vspace{-19pt}

\subsection{Proposed Algorithm: \algicatg}
\label{subsec:algicat_g}

\textbf{Main idea.} By and large, we use the same EXP3 based algorithm as proposed for the case of \emph{independent availabilities}, the only difference lies in using a different empirical estimate 
\begin{equation}
   \bar q_t (i) := \frac{1}{t} \sum_{\tau = 1}^{t} q_t^{S_\tau}(i) \,.
   \label{eq:barqt_general_def}
\end{equation} 
In hindsight, the above estimate $\bar q_t (i)$ is equal to the expectation $\E_{S \sim \hat P_t}[q_t^{S}]$, i.e., $\bar q_t (i) = \sum_{S \in 2^{[K]}} \hat P(S)q_t^{S}(i)$, where $\hat P_t(S): = \frac{1}{t}\sum_{\tau = 1}^t\1(S_\tau = S)$ is the empirical probability of set $S$ at time $t$.
The rest of the algorithm proceeds the same as Alg.~\ref{alg:icat}, the complete description is given in Alg.~\ref{alg:icatg}.

\subsection{Regret Analysis}
\label{subsec:reg_bnd_g}
We first analyze the concentration of $\bar q_t(i)$--the empirical probability of playing item $i$ at any round $t$, and the result goes as follows:

%\iffalse %%%%%%%%%%%%%%%%%%%%
\begin{restatable}[Concentration of $\bar q_t(i)$]{lem}{concbarqg}
\label{lem:conc_barq_g}
Let $t \in [T]$. Let $q^*_t(i) = \E_{S \sim \a}\big[ q_{t}^{S}(i) \big]$, and define $\bar q_t(i)$ as in Equation~\eqref{eq:barqt_general_def}. Then, with probability at least $(1-\delta)$, 
\[
%\label{eq:barq_conc}
|q^*_t(i) - \bar q_t(i)| \le \sqrt{\frac{2^{K+1}}{t} \ln \frac{2^K}{\delta}} + \frac{2^{K+1}}{3t} \ln \frac{2^K}{\delta} \,,
\]
for all $i \in [K]$.
\end{restatable}
%\fi %%%%%%%%%%%%%%%%%%%%%%%%%

Using Lem.~\ref{lem:conc_barq_g} we now analyze the regret bound of Alg.~\ref{alg:icatg}.

\begin{restatable}[\algicatg: Regret Analysis]{thm}{ubicatg}
\label{thm:reg_icat_g}
Let $T \geq 1$. Suppose we set $\eta = \sqrt{(\log K)/(KT)}$, $\delta = 2^K/T^2$, and $\lambda_t =  \sqrt{(2^{K+1}/{t}) \ln ({2^K}/{\delta})} + 2^{K+1} \ln ({2^K}/{\delta})/(3t)$. Then, the regret incurred by \algicatg\, (Alg.~\ref{alg:icatg}) can be bounded as:
\[
 R_T \le K \sqrt{2^{K+4} T \log T} + K2^{K+3}(\log T)^2 \,.
\]

Furthermore, the per-round space and time complexities of the algorithm are $O(tK)$. 
\end{restatable}

\begin{proof}\textbf{\hspace{-3pt}(sketch) }
The proof proceeds almost similar to the proof of Thm.~\ref{thm:reg_icat} except now the corresponding version of the main lemmas, aka. Lem.~\ref{lem:first_term},\ref{lem:second_term}, and~\ref{lem:var_term} are satisfied but for 
$
\lambda_t = \sqrt{\frac{2^{K+1}}{t} \ln \frac{2^K}{\delta}}+ \frac{2^{K+1}}{3t} \ln \frac{2^K}{\delta} \,,
$
since here we need to use the concentration Lem.~\ref{lem:conc_barq_g} instead of Lem.~\ref{lem:conc_barq_p}. 

Similar to the proof of Thm.~\ref{thm:reg_icat} and following the same notation, we first combine claims from Lem.~\ref{lem:first_term},~\ref{lem:second_term}  to get:
\begin{align*}
&R_{T}(\text{\algicatg}) =  \sum_{t=1}^{T}\E\Big[\ell_t(i_t) - \ell_t(\pi^*(S_t))\Big]\\
& \leq \sum_{t=1}^{T}\E\big[\big< \q_t^S,\hat \ell_t \big> + 2K \lambda_t + \frac{2\delta}{\lambda_t} \big] - \E[\hat \ell_t(\pi^*(S_t))] \\
& \overset{(a)}{\leq}  \frac{\log K}{\eta} +  \eta \sum_{t = 1}^T \E\bigg[\sum_{k \in S} q_t^S(k)\hat \ell_t(k)^2 \bigg]  \\
&    \hspace*{4cm} +  2\sum_{t = 1}^{T} \big(K\lambda_t + \frac{\delta}{\lambda_t}\big)\\
& \overset{(b)}{\leq} \frac{\log K}{\eta} +  \eta KT +  \sum_{t = 1}^{T} \big(2 K\lambda_t + \frac{2\delta}{\lambda_t} + \frac{\eta \delta}{\lambda_t^2} \big)\\
& \leq  \frac{\log K}{\eta} +  \eta KT +  \sum_{t = 1}^{T} \big(2 K\lambda_t + \frac{3 \delta}{\lambda_t^2} \big), 
\end{align*}
where Inequality (a) and (b) respectively follow from \eqref{eq:RegretEXP31} and Lem.~\ref{lem:var_term}. The last inequality holds because $\eta \leq 1$ and $\lambda_t \leq 1$.
To conclude the proof, it now only remains to compute the sums and to choose the parameters $\delta = 2^K / T^2$ and $\eta = \sqrt{(\log K)/KT}$. 
Using $\lambda_t \geq \sqrt{2^{K+1} /t}$, we have
$
\sum_{t=1}^T \frac{\delta}{\lambda_t^2} \leq  \frac{\delta T^2}{2^{K+1}} \leq 1,
$ 
and since $\log(2^K/\delta) = 2 \log T$, we further have
\[
\lambda_t = \sqrt{\frac{2^{K+1}}{t} \ln T }+ \frac{2^{K+1}}{3t} \ln T
\]
which entails:
\begin{align*}
\sum_{t=1}^T \lambda_t & \leq \sqrt{2^{K+1} T \log T} + \frac{2^{K+1}}{3} (\log T)(1+ \log T) \\
& \leq \sqrt{2^{K+1} T \log T} + 2^{K+2} (\log T)^2.
\end{align*}

Finally, substituting $\eta$ and the above bounds in the regret upper-bound yields the desired result.

\textbf{Complexity analysis.} The only difference with Alg.~\ref{alg:icat} lies in computing $\bar q_t(i)$. Following a similar argument given for proving the computational complexity of Thm.~\ref{thm:reg_icat_e}, this can also be performed with a computational cost of $O(tK)$. Yet, now the algorithm specifically needs to keep in memory the empirical distribution of $S_1,\dots,S_t$ and thus a space complexity of $O(K + \min\{tK,2^K\})$ is required. 
\end{proof}

Our regret bound in Thm.~\ref{thm:reg_icat_g} has the optimal $\sqrt{T}$ dependency---to the best of our knowledge, \algicatg \,(Alg~\ref{alg:icatg}) is the first \emph{computationally efficient} algorithm to achieve  $O(\sqrt T)$ guarantee for the problem of \emph{Sleeping-Bandits}. Of course the EXP4 algorithm is known to attain the optimal $O(\sqrt {KT})$ regret bound, however it is computationally infeasible \cite{kleinberg+10} due to the overhead of maintaining a combinatorial policy class. 

Yet, on the downside, it is worth pointing out that the regret bound of Thm.~\ref{thm:reg_icat_g} only provides a sublinear regret in the regime $2^K \leq O(T)$, in which case algorithms such as EXP4 can be efficiently implemented. However, we still believe our algorithm to be an interesting contribution because it completes another side of the computational-performance trade-off. It is possible to move the exponential dependence on the number of experts from the computational complexity to the regret bound. 

Another argument in favor of this algorithm is that it provides an efficient alternative algorithm to EXP4 with regret guarantees in the regime $2^K \leq O(T)$. In the other regime, though we could not prove any meaningful regret guarantee, Alg.~\ref{subsec:algicat_g} performs very well in practice as shown by our experiments. We believe the $2^K$ constant in the regret to be an artifact of our analysis. However, removing it seems to be highly challenging due to dependencies between $q_t^{S}$ and $S_1,\dots,S_t$. An analysis of the concentration of $\bar q_t$ (defined in~\eqref{eq:barqt_general_def}) to $q_t^*$ without exponential dependence on $K$ proved to be particularly complicated. We leave this question for future research.

%%%%%%%%%%%%%%%%%%%%%%%%%%%%%

%!TEX root = sleeping_icml20.tex

\begin{figure*}[t]
	\begin{center}
		\includegraphics[width=0.33\textwidth]{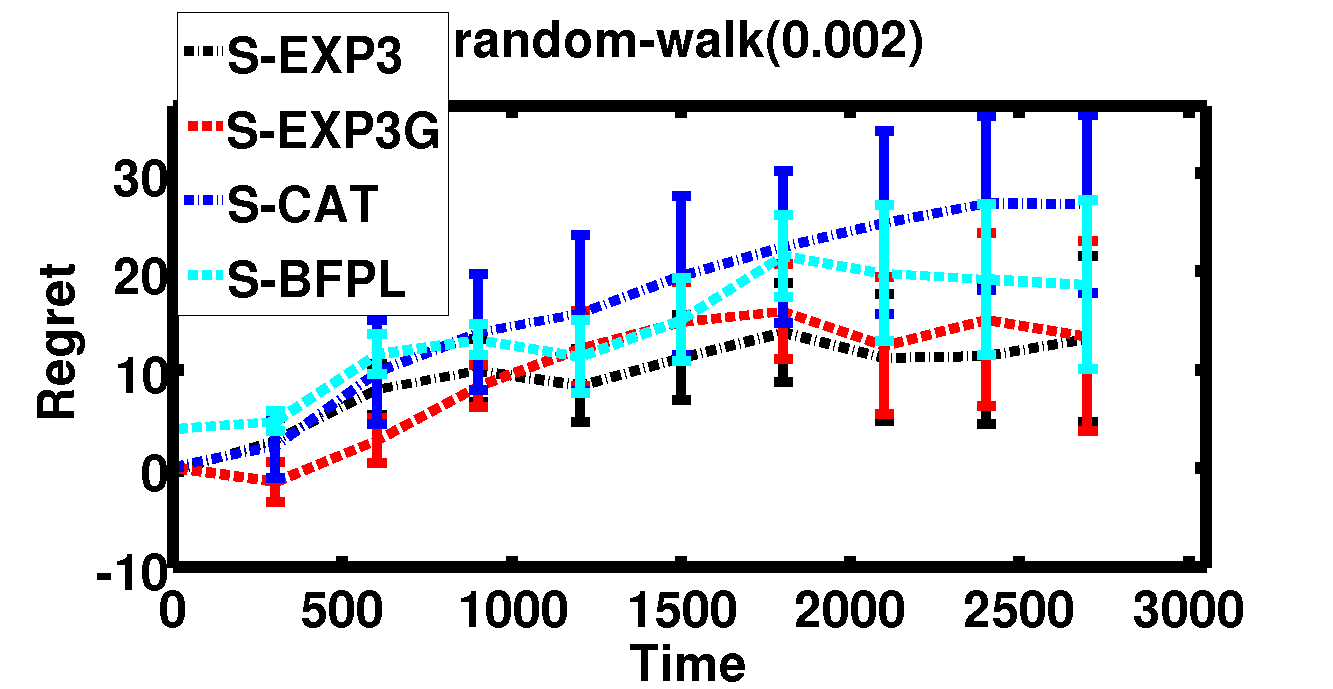}
		\includegraphics[width=0.33\textwidth]{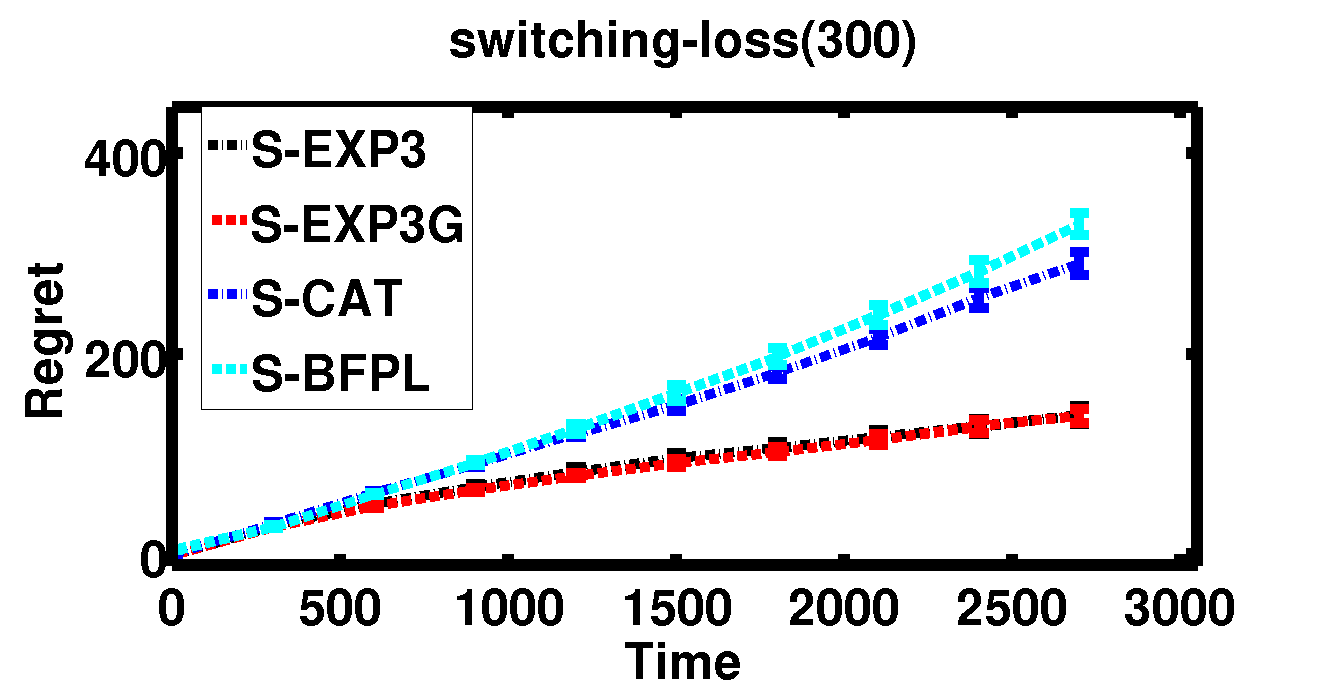}
		\includegraphics[width=0.33\textwidth]{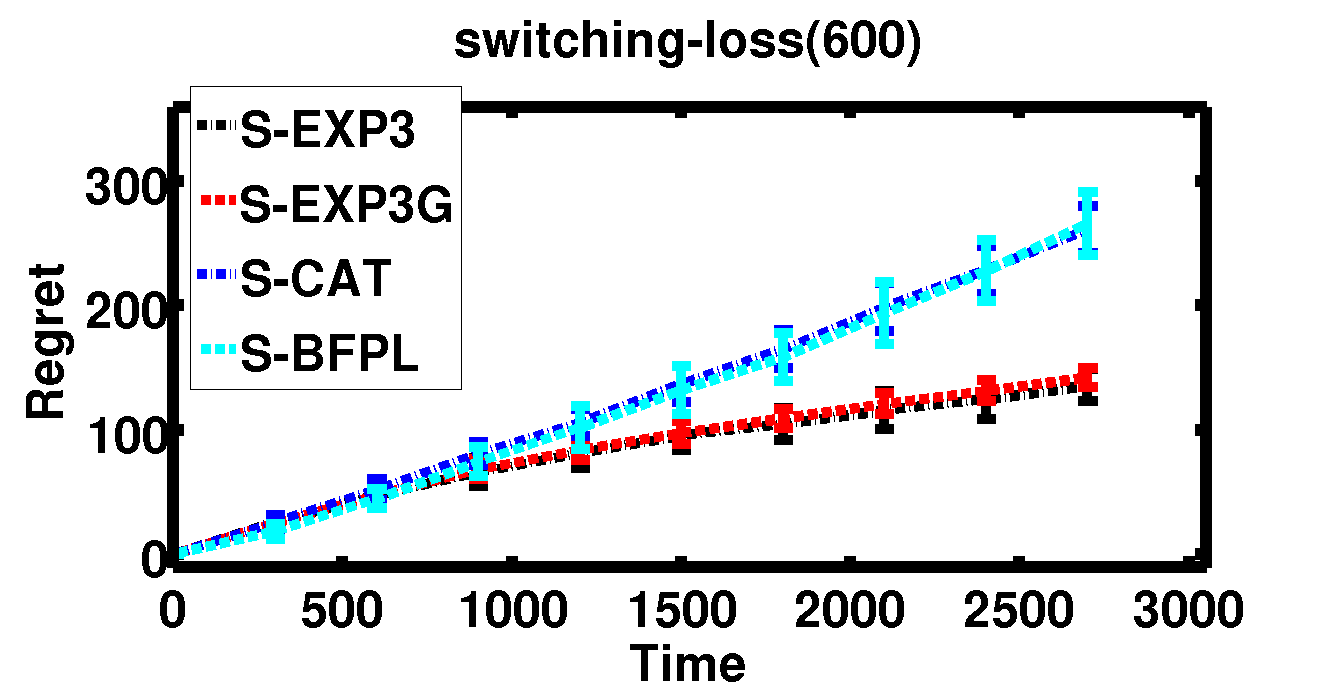}
		\vspace{-20pt}
		\caption{Regret vs Time: Independent availabilities}
		\label{fig:reg_indep}
	\end{center}
\end{figure*}

\section{Experiments}
\label{sec:expts}
In this section we present the empirical evaluation of our proposed algorithms (Sec.~\ref{sec:algo_indep} and~\ref{sec:algo_dep}) comparing their performances with the two existing \textit{sleeping bandit} algorithms that apply to our problem setting, i.e. for adversarial losses and stochastic availabilities. Thus we report the comparative performances of the following algorithms: 

\begin{enumerate}[nosep]
	\item \algicat: Our proposed Alg.~\ref{alg:icat} (the efficient version as described in Sec.~\ref{subsec:alg_eff}).  
	
	\item \algicatg: Our proposed Alg.~\ref{alg:icatg}. 
	 
	\item \algscat: The algorithm proposed by \cite{neu14} (precisely their Algorithm for semi-bandit feedback in Sec. $4.3$).
	
	\item \algkand: The algorithm proposed by \cite{kanade09} (see Fig. 3, BSFPL algorithm, Sec. $2$).
\end{enumerate}

\textbf{Performance Measures.} In all cases, we report the cumulative regret of the algorithms for $T = 5000$ time steps, each averaged over $50$ runs.
In the following subsections, we analyze our experimental evaluations for both independent and general (non-independent) availabilities.

\subsection{Independent Availabilities}
\label{subsec:expts_indep}

In this case the item availabilities are assumed to be independent at each round (description in  Sec.~\ref{sec:prob}).

\textbf{Environments.} 
We consider $K=20$ and generate the probabilities of item availabilities $\{a_{i}\}_{i \in [K]}$ independently and uniformly at random from the interval $[0.3,0.9]$.
%Recall we assumed the loss sequence to be oblivious to the availabilities, towards which 
We use the following loss generation techniques: 
$(1)$ \emph{Switching loss or SL($\tau$)}. We generate the loss sequence such that the best performing expert changes after every $\tau$ length epochs. 
$(2)$ \emph{Markov loss or ML(p)}. Similar to the setting used in \cite{neu14}, losses for each arm are constructed as random walks with Gaussian
increments of standard deviation $p$, initialized uniformly on $[0, 1]$ such that losses outside $[0, 1]$ are truncated. The explicit values used for $\tau$ and $p$ are specified in the corresponding figures.
The algorithm parameters $\eta, \lambda_t, \delta$ are set as defined in Thm7. %Note we implemented the computationally efficient version of Alg1 for Sleeping-Exp3 as mentioned in Sec5. Similarly, for general availabilities, we set the parameters as defined in Thm9.

% \begin{figure*}[tp]
% 	\begin{center}
% 		\includegraphics[trim={0cm 0.5cm 0cm 2},clip,scale=0.28,width=0.3\textwidth]{./plots/reg_ind/K=10_rk.png}
% 		\hspace{5pt}
% 		\includegraphics[trim={0cm 0.5cm 0cm 2},clip,scale=0.28,width=0.3\textwidth]{./plots/reg_ind/K=10_sw.png}
% 		\hspace{5pt}
% 		\includegraphics[trim={0cm 0.5cm 0cm 2},clip,scale=0.28,width=0.3\textwidth]{./plots/reg_ind/K=10_sw2.png}
% 		\vspace{-10pt}
% 		\caption{Regret vs Time: Independent availabilities}
% 		\label{fig:reg_indep}
% 		\vspace{0pt}
% 	\end{center}
% \end{figure*}

\textbf{Remarks.} From Fig.~\ref{fig:reg_indep} it clearly shows that regret bounds of our proposed algorithm \algicat \, and \algicatg\, outperform the other two due to their orderwise optimal $O(\sqrt T)$ regret performance (see Thm.~\ref{thm:reg_icat} and~\ref{thm:reg_icat_g}). In particular, \algkand\, performs the worst due to its initial $O(T^{4/5})$ exploration rounds and uniform exploration phases thereafter. \algscat\, gives a much competitive regret bound compared to \algkand\, however, still lags behind due to the $O(T^{2/3})$ regret guarantee (see Sec. $3.1$ for a detailed explanation).

\begin{figure}[!h]
	\begin{center}
		%\hspace*{-10pt}	
		\includegraphics[width=0.40\textwidth]{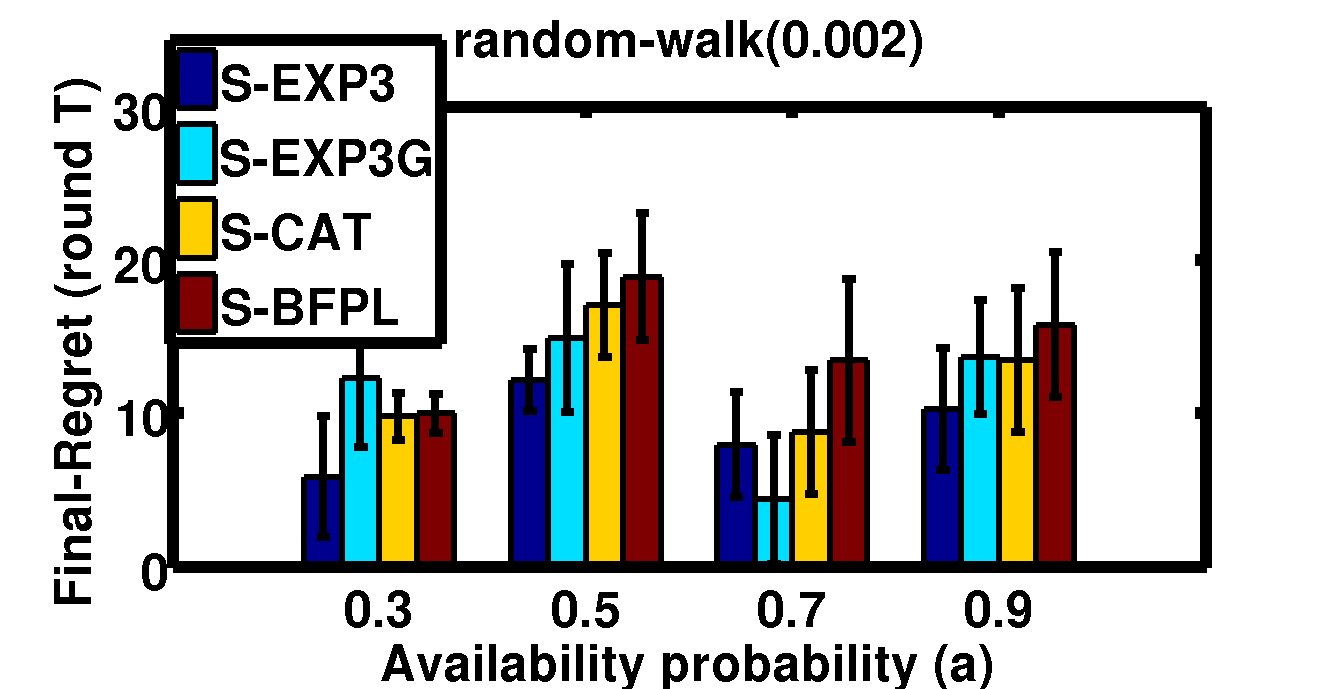}
		%\hspace*{-10pt}	
		\includegraphics[width=0.40\textwidth]{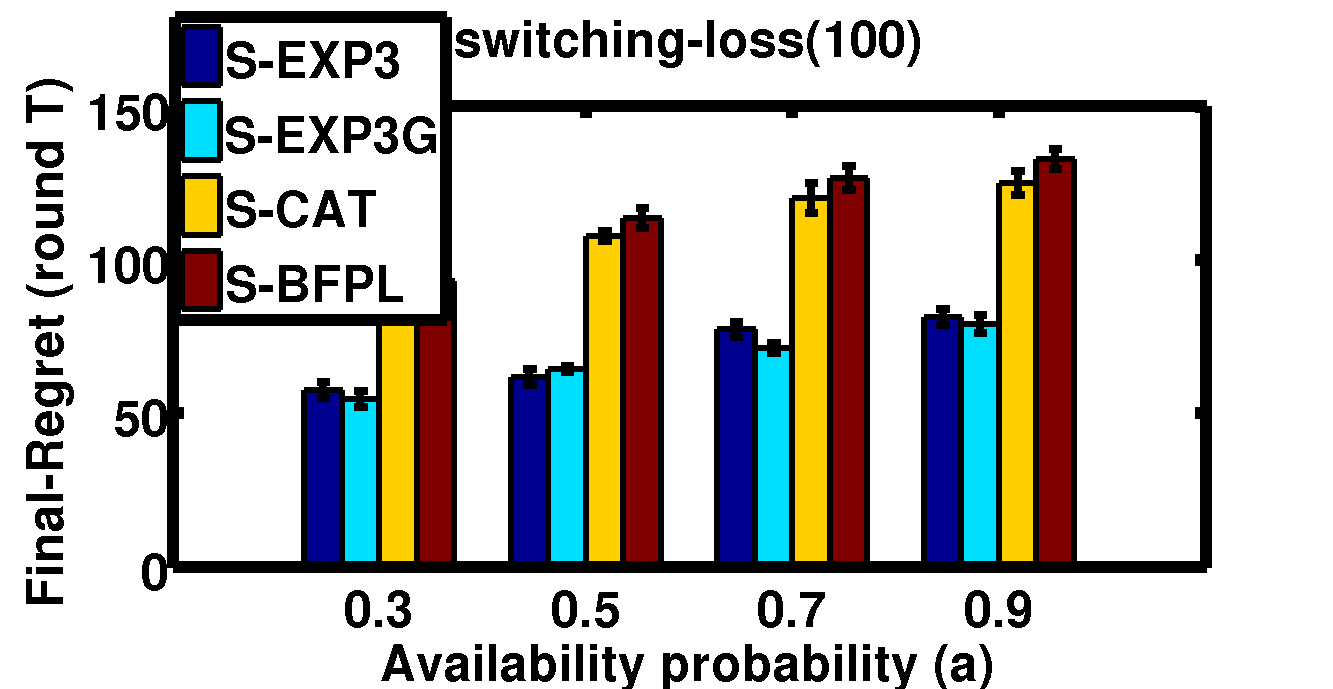}
		\vspace{-6pt}
		\caption{Final regret (at round $T$) vs availability probabilities($(a)$)}
		\label{fig:reg_vs_p}
	\end{center}
\end{figure}

\subsection{Regret vs Varying Availabilities.} We next conduct a set of experiments to compare the regret performances of the algorithms with varying availability probabilities: For this we assign same availability $a_i = a \in [0.1]$ to every item $i \in [K]$ for $a = 0.3, 0.5, 0.7, 0.9$ and plot the final cumulative regret of each algorithm.% after $T$ rounds.

% \begin{figure}[tp]
% 	\begin{center}
% 		\hspace{-5pt}
% 		\includegraphics[trim={0cm 0.5cm 0cm 2},clip,scale=0.28,width=0.24\textwidth]{./plots/finreg_vs_p/rnd002.png}
% 		\hspace{-5pt}
% 		\includegraphics[trim={0cm 0.5cm 0cm 2},clip,scale=0.28,width=0.24\textwidth]{./plots/finreg_vs_p/swch.png}
% 		%\hspace{5pt}
% 		%\includegraphics[trim={0cm 0.5cm 0cm 2},clip,scale=0.28,width=0.3\textwidth]{./plots/phold.png}
% 		\vspace{-10pt}
% 		\caption{Final regret (at round $T$) vs availability probabilities($(a)$)}
% 		\label{fig:reg_vs_p}
% 		\vspace{0pt}
% 	\end{center}
% \end{figure}

\textbf{Remarks}
From Fig.~\ref{fig:reg_vs_p}, we again note our algorithms outperform the other two by a large margin for almost every $p$.
%when p is very low, there are few or no arms to choose from. In this case the problems are easy by design and all algorithms suffer low regret. As p increases, the policy space starts to blow up and the problem becomes more difficult. When p approaches one, it collapses into the set of single arms and the problem gets easier again. The behavior of our algorithm as well as \algscat\, follows this trend. On the other hand, 
The performance of BSFPL is worse, it steadily decreases with increasing availability probability due to the explicit $O(T^{4/5})$ exploration rounds in the initial phase of BSFPL, and even thereafter it keeps on suffering the loss of the uniform policy scaled by the exploration probability. 

\begin{figure*}[t]
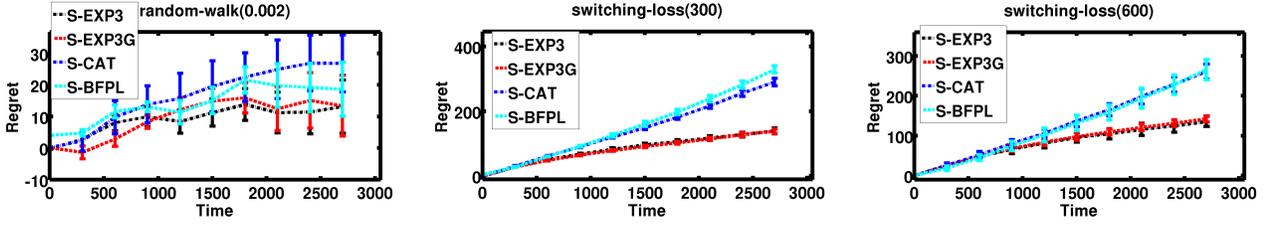

	\begin{center}
		\includegraphics[width=0.33\textwidth]{K=10_rk.png}
		\includegraphics[width=0.33\textwidth]{K=10_sw.png}
		\includegraphics[width=0.33\textwidth]{K=10_sw2.png}
		\vspace{-20pt}
		\caption{Regret vs time: General availabilities}
		\label{fig:reg_dep}
	\end{center}
\end{figure*}

\begin{figure*}[t]
	\begin{center}
		\includegraphics[width=0.33\textwidth]{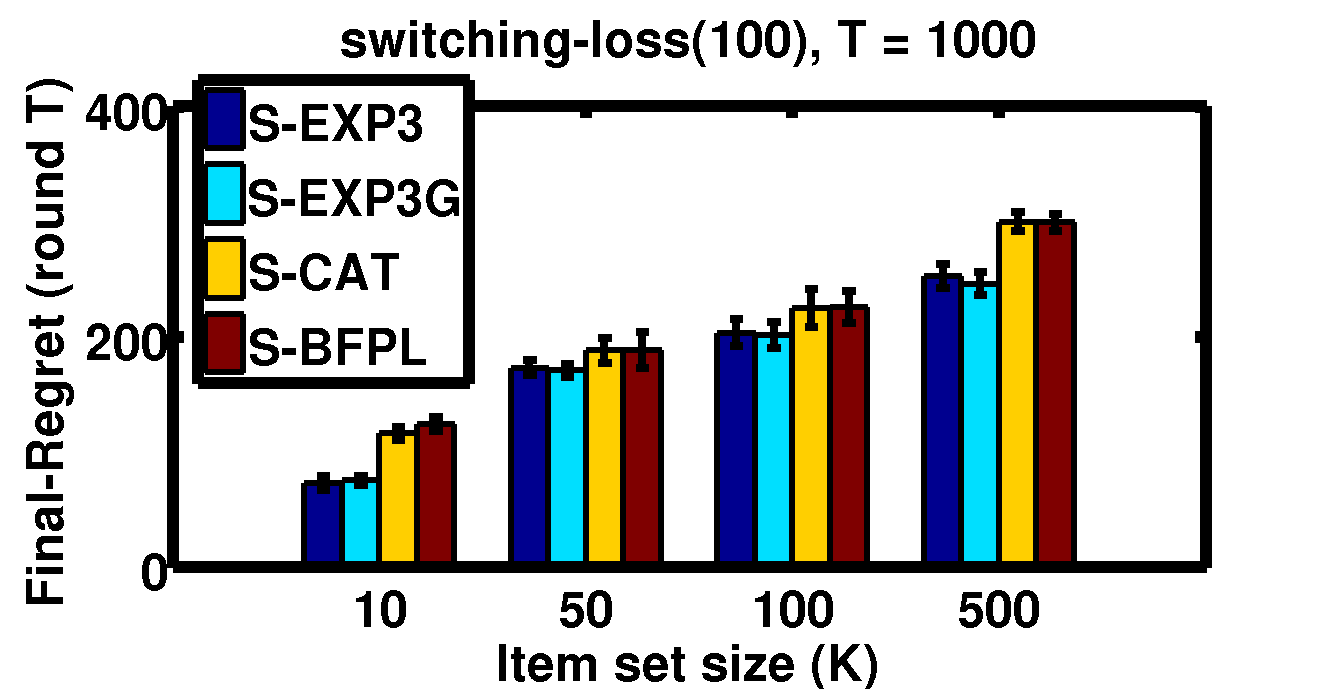}
		\includegraphics[width=0.33\textwidth]{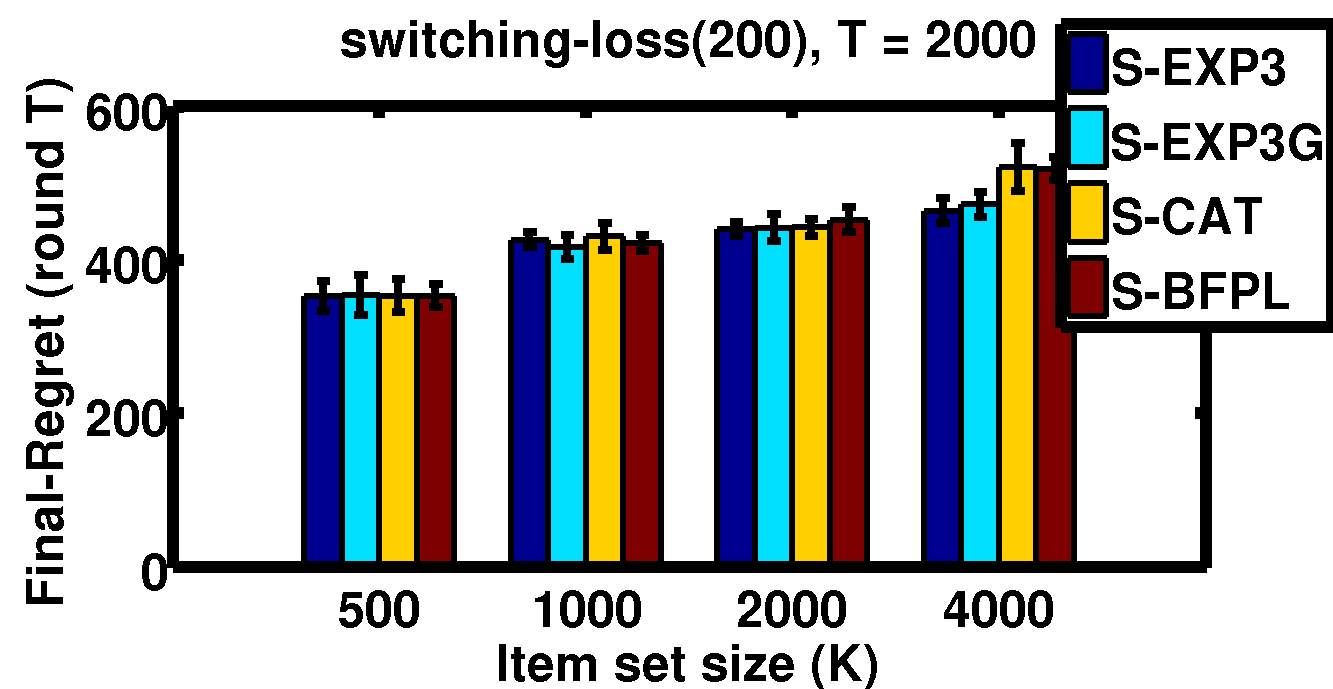}
		\includegraphics[width=0.33\textwidth]{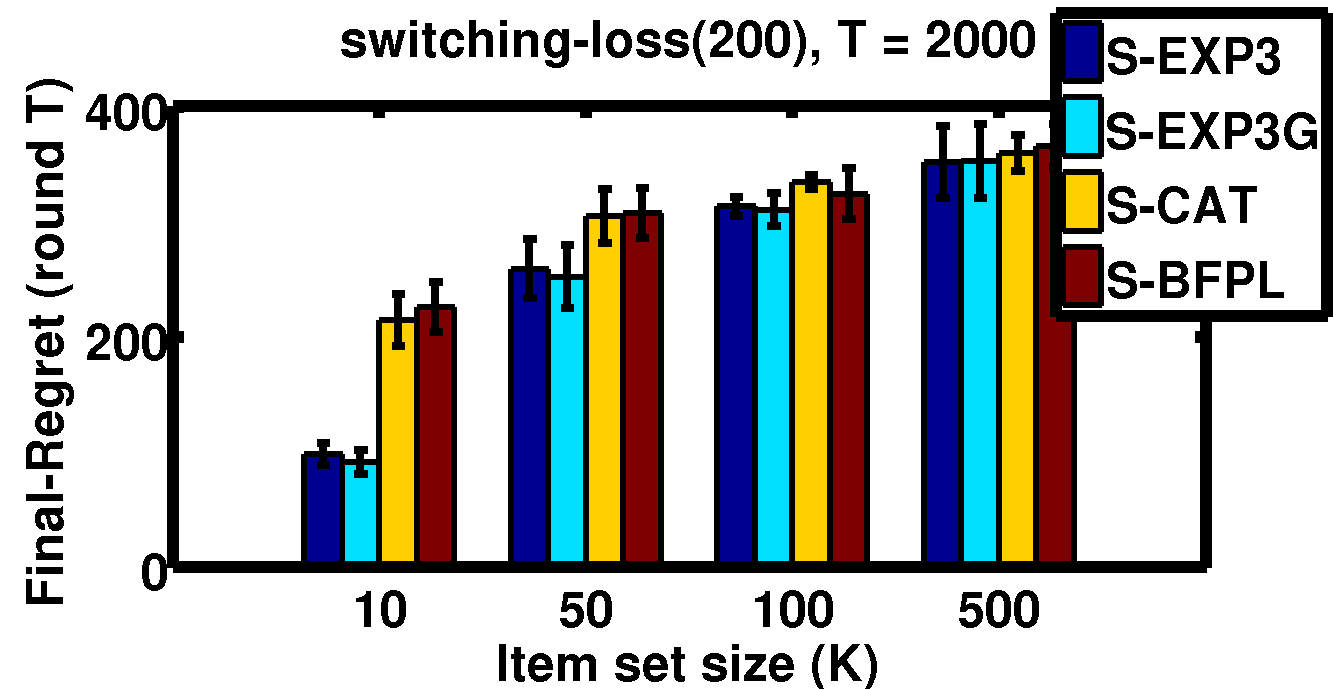}
		\vspace{-20pt}
		\caption{Final regret (at round $T$) vs item size $(K)$}
		\label{fig:reg_vs_K}
		\vspace{0pt}
	\end{center}
\end{figure*}

\subsection{Correlated (General) Availabilities}
We now assess the performances when the availabilities of items are dependent (description in Sec.~\ref{sec:prob}).

\textbf{Environments.}
To enforce dependencies of item availabilities we generate each set $S_t$ by drawing a random sample from a Gaussian$(\bmu, \Sigma)$ such that $\bmu_i = 0,\, \forall i \in [K]$, and $\Sigma$ is some random $K\times K$ positive definite matrix, e.g. block diagonal matrix with strong correlations among certain groups of items. More precisely, 
at each round $t$, we first sample a random $K$-vector, say $v_t$, from Gaussian$(\bmu, \Sigma)$ and we set $S_t = {i \in [K]|v_t(i)>0}$, i.e. $S_t$ includes all those items whose corresponding coordinates are non-negative--this thus enforces item dependencies in the resulted $S_t$ if $\Sigma$ is block diagonal (or any correlation matrix). 
%positive definite matrix such that for $i,j$-th entry $\Sigma(i,j) = \Sigma(j,i) \in [0,1]$, and $\Sigma(i,i) = 1, \forall i,j \in [K]$, $S_t$ is considered to be the set of items whose value exceeds $0.5$.
To generate the loss sequences, we use similar techniques described in Sec.~\ref{subsec:expts_indep}.
The algorithm parameters are set as defined in Thm9. 

% \begin{figure*}[tph]
% 	\begin{center}
% 		\includegraphics[trim={0cm 0.5cm 0cm 2},clip,scale=0.28,width=0.3\textwidth]{./plots/reg_gen/K=10_rk.png}
% 		\hspace{5pt}
% 		\includegraphics[trim={0cm 0.5cm 0cm 2},clip,scale=0.28,width=0.3\textwidth]{./plots/reg_gen/K=10_sw.png}
% 		\hspace{5pt}
% 		\includegraphics[trim={0cm 0.5cm 0cm 2},clip,scale=0.28,width=0.3\textwidth]{./plots/reg_gen/K=10_sw2.png}
% 		\vspace{-10pt}
% 		\caption{Regret vs time: General availabilities}
% 		\label{fig:reg_dep}
% 		\vspace{0pt}
% 	\end{center}
% \end{figure*}

\textbf{Remarks.} From Fig.~\ref{fig:reg_dep} one can again verify the superior performance of our algorithms over \algscat\, and \algkand, however the effect is only visible for large $T$ as for smaller time steps $t$, the $O(2^K)$ terms dominates the regret performance, but as $t$ shoots higher our optimal $O(\sqrt T)$ rate outperforms the suboptimal $O(T^{2/3})$ and $O(T^{4/5})$ rates of \algscat\, and \algkand\, respectively.

\subsection{Regret vs Varying Item-size $(K)$.} Finally we also conduct a set of experiments changing the item set size $K$ over a wide range ($K = 10$ to $4000$). We report the final cumulative regret of all algorithms vs. $K$ for different switching loss sequence for both independent and general availabilities, as specified in Fig.~\ref{fig:reg_vs_K}. 

% \begin{figure}[tph]
% 	\begin{center}
% 		%\hspace{-5pt}
% 		\includegraphics[trim={0cm 0.5cm 0cm 2},clip,scale=0.28,width=0.3\textwidth]{./plots/finreg_vs_K/swch_Ksml.png}
% 		\hspace{5pt}
% 		\includegraphics[trim={0cm 0.5cm 0cm 2},clip,scale=0.28,width=0.3\textwidth]{./plots/finreg_vs_K/swch_Kbig.png}
% 		\hspace{5pt}
% 		\includegraphics[trim={0cm 0.5cm 0cm 2},clip,scale=0.28,width=0.3\textwidth]{./plots/finreg_vs_K/swch_Ksml_gen.png}
% 		\vspace{-10pt}
% 		\caption{Final regret (at round $T$) vs item size $(K)$}
% 		\label{fig:reg_vs_K}
% 		\vspace{0pt}
% 	\end{center}
% \end{figure*}

\textbf{Remark.} Fig.~\ref{fig:reg_vs_K} shows that the regret of each algorithm increases with $K$, as expected. As before the other two baselines perform suboptimally in comparison to our algorithms, however the interesting thing to note is the relative performance of \algicat \, and \algicatg---as per Thm.~\ref{thm:reg_icat} and~\ref{thm:reg_icat_g}, \algicat\, must outperform \algicatg\, with increasing $K$, however the effect does not seem to be so drastic experimentally, possibly revealing the scope of improving Thm.~\ref{thm:reg_icat_g} in terms of a better dependency in $K$.

%!TEX root = sleeping_icml20.tex
\section{Conclusion and Future Work}
\label{sec:concl}
We have presented a new approach that brought an improved rate for the setting 
of sleeping bandits with adversarial losses and stochastic availabilities 
including both minimax and instance-dependence guarantees. 
While our bounds guarantee a regret of $\tilde O(\sqrt{T}),$  there 
are several open questions before the studied setting can be considered as closed.
Firstly, for the case of independent availabilities, we provide a regret guarantee of
$\tilde O(K^2\sqrt{T})$, leaving open whether $\tilde O(\sqrt{KT})$ is possible as
in the standard non-sleeping setting. 
Secondly, while we provided computationally efficient  (i.e., with per-round complexity of order $O(tK)$)  
 \algicat, for the case of general availabilities and provided instance
 dependent regret guarantees for it, the worst case regret guarantee
 still amounts to $\tilde O(\sqrt{2^KT}).$ Therefore, it is still unknown 
 if for the general availabilities we can get an algorithm that would be both computationally 
 efficient and have  $\tilde O(\text{poly}(K)\sqrt{T})$ regret guarantee in the worst case.
We would like to point out that the new techniques could be potentially 
used to  provide new algorithms and guarantees in settings with similar challenges
as in sleeping bandits, such as rotting or dying bandits.
Finally, having algorithms for sleeping bandits with $\tilde O(\sqrt{T})$  regret guarantees, opens a way to deal with sleeping constraints in more challenging 
 structured bandits with large or infinite number of arms and having the regret 
 guarantee depend not on number of arms but rather some effective dimension 
 of the arms' space.

%\input{other_approaches.tex}

\section*{Acknowledgements}
We thank the anonymous reviewers for their valuable suggestions. 
The research presented was supported  French National Research Agency project  BOLD (ANR-19-CE23-0026-04) and by European CHIST-ERA project DELTA. We also wish to thank Antoine Chambaz and Marie-H\'el\`ene Gbaguidi for supporting Aadirupa's internship at Inria, and Rianne de Heide for the thorough proofreading. 

%\textbf{Do not} include acknowledgements in the initial version of the paper submitted for blind review.


%\newpage
\bibliographystyle{icml2020}
\bibliography{bib_sleeping_bandits}  % put name of your .bib file here

\newpage

\appendix
%!TEX root = sleeping_icml20.tex
\appendix
\onecolumn{
	
\section*{\centering\Large{Supplementary: Improved Sleeping Bandits with Stochastic Actions Sets \\ and Adversarial Rewards}}
\vspace*{1cm}
	
\allowdisplaybreaks
	
\section{Appendix for Sec.~\ref{sec:algo_indep}}

\subsection{Proof of Lem.~\ref{lem:conc_barq_p}} 
\label{app:conc_icatp}

\concbarqp*

\begin{proof}
%Alernative proof: From Bernstein inequality:
Let $t \in [T]$ and $\delta \in (0,1)$. We start by noting the concentration of $\hat a_{ti} = \frac{1}{t} \sum_{\tau =1}^t \1(i \in S_\tau)$ to $a_i$ for all $i \in [K]$. By Bernstein's inequality together with a union bound over $i=1,\dots,K$: with probability at least $1-\delta$, for all $i \in [K]$
\begin{equation}
	 |a_i - \hat a_{ti}| < \sqrt{\frac{2 a_i(1-a_i)\ln \frac{K}{\delta}}{t}} + \frac{2\ln \frac{K}{\delta}}{3t} \,.
	 \label{Bernstein}
\end{equation}

Then, $\beta = 2 \log (K/\delta)$, using the definitions of $P_{\hat \a_t}(S)  = \Pi_{i =1}^{K}\hat a_{ti}^{\1(i \in S)} (1-\hat a_{ti})^{1-\1(i \in S)}$ and $P_{\a}(S) = \Pi_{i =1}^{K} a_{i}^{\1(i \in S)} (1-a_{i})^{1-\1(i \in S)}$, we get
	\begin{align}
	\label{eq:concP}
	\nonumber |P_{\hat \a_t}(S) - P_{\a}(S)| & = P_{\a}(S)\bigg|\frac{P_{\hat \a_t}(S)}{P_{\a}(S)} - 1\bigg| = P_{\a}(S)\bigg| \Pi_{i =1}^{K}\bigg(\frac{\hat a_{ti}}{a_i}\bigg)^{S_i} \bigg(\frac{1-\hat a_{ti}}{1-a_i}\bigg)^{1-S_i} - 1\bigg| \nonumber \\
	& = P_{\a}(S)\bigg| \Pi_{i =1}^{K}\bigg(\frac{\hat a_{ti} - a_i}{a_i} + 1\bigg)^{S_i} \bigg(\frac{a_i-\hat a_{ti}}{1-a_i} + 1\bigg)^{1-S_i} - 1\bigg| \nonumber \\
	& \leq P_{\a}(S) \bigg( \Pi_{i =1}^{K}\bigg(\frac{|\hat a_{ti} - a_i|}{a_i} + 1\bigg)^{S_i} \bigg(\frac{|a_i-\hat a_{ti}|}{1-a_i} + 1\bigg)^{1-S_i} - 1\bigg)  \,,
	\end{align}
	where the last inequality is because for any $\epsilon_1,\dots,\epsilon_K \in [-1,1]^K$
	\begin{align*}
		\left| \prod_{i=1}^K (1+\epsilon_i) - 1 \right| 
			& = \max\left\{\prod_{i=1}^K (1+\epsilon_i) - 1, 1-\prod_{i=1}^K (1+\epsilon_i)  \right\} \\
			& \leq \max\left\{\prod_{i=1}^K \big(1+|\epsilon_i|\big) - 1, 1-\prod_{i=1}^K \big(1-|\epsilon_i|\big)  \right\} \\
			& \leq \prod_{i=1}^K \big(1+|\epsilon_i|\big) - 1 \,.
	\end{align*}

	Hence, denoting
	\[
	Z_i := \bigg(1 + \sqrt{\frac{\beta(1-a_i)}{a_i t}} + \frac{\beta}{3a_i t} \bigg)^{S_i} \bigg(1 + \sqrt{\frac{\beta a_i}{(1-a_i) t}} + \frac{\beta}{3(1-a_i) t}\bigg)^{1-S_i} \,,
	\]
	and using Bernstein's inequality~\eqref{Bernstein}, with probability $1-\delta$, we have
	\begin{equation}
		\label{eq:Phat_trueP}
		|P_{\hat \a_t}(S) - P_{\a}(S)|  \leq P_{\a}(S) \left( \prod_{i=1}^K Z_i - 1\right) \,.
	\end{equation}

	Therefore,
	\begin{align}
	| q^*_t(i) - \bar q_t(i)|  & =  \bigg| \sum_{S \in 2^{[K]}}q_t^S(i) \big(P_{\a}(S) - P_{\hat \a_t}(S)\big) \bigg| \nonumber \hspace*{1cm} \leftarrow \text{By definition~\eqref{eq:barqt_def}}  \\
	& \le  \sum_{S \in 2^{[K]}}q_t^S(i) \big|P_{\a}(S) - P_{\hat \a_t}(S)\big|  \nonumber \\
	&  \le \sum_{S \in 2^{[K]}}  P_{\a}(S) \left( \prod_{i=1}^K Z_i - 1\right)  \hspace*{1.7cm} \leftarrow \text{From~\eqref{eq:Phat_trueP} and $|q_t^S(i)| \leq 1$} \nonumber \\
	&  = \E_{S \sim \P_{\a}} \left[ \prod_{i=1}^K Z_i \right] - 1 \,. \label{eq:error_barq}
	\end{align}
	But by independence of the $Z_i$, we know
	\begin{equation}
	\label{eq:expectedZi}
	\E_{S \sim \P_{\a}} \left[ \prod_{i=1}^K Z_i \right] 
	= \prod_{i=1}^K \E_{S \sim \P_{\a}}\big[Z_i\big]   \,.
	\end{equation}
	Now, computing each expectation
	\begin{align}
	\E_{S \sim \P_{\a}}\big[Z_i\big] & = \E_{S_i \sim a_i} \left[ \bigg(1 + \sqrt{\frac{\beta(1-a_i)}{a_i t}} + \frac{\beta}{3a_i t} \bigg)^{S_i} \bigg(1 + \sqrt{\frac{\beta a_i}{(1-a_i) t}} + \frac{\beta}{3(1-a_i) t}\bigg)^{1-S_i} \right]   \nonumber \\
	& = a_i \Big(1 + \sqrt{\frac{\beta (1-a_i)}{a_i t}} + \frac{\beta}{3a_i t} \Big) + (1-a_i) \Big(1 + \sqrt{\frac{\beta a_i}{(1-a_i) t}} + \frac{\beta}{3(1-a_i) t}\Big)  \nonumber \\
	& \leq 1 + 2 \sqrt{\frac{\beta a_i(1-a_i) }{t}} + \frac{2\beta}{3t}	\label{eq:instancedependent_bound}  \\ 
	& \leq 1 +  \sqrt{\frac{\beta}{t}} + \frac{2\beta}{3t}
	 \,. \nonumber 
	\end{align}
	
	Therefore, combining with Inequalities~\eqref{eq:error_barq} and~\eqref{eq:expectedZi}, it yields using $1+x \leq e^x$ for $x\geq 0$,
	\[
	| q^*_t(i) - \bar q_t(i)|  \leq \Big(1 +  \sqrt{\frac{\beta}{t}} + \frac{2\beta}{3t}\Big)^{K} - 1  \leq e^{K \sqrt{\frac{\beta}{t}} + \frac{2K\beta}{3t}} - 1 \,.
	\]
	
	Then, assuming $t \geq 25 \beta K^2 / 9$, we have
	\[
	 K \sqrt{\frac{\beta}{t}} + \frac{2K\beta}{3t} \leq 1
	\]
	which implies since $e^x \leq 1 + 2x $ for all $0\leq x \leq 1$ and replacing $\beta = 2 \log (K/\delta)$
	\[
	| q^*_t(i) - \bar q_t(i)|  \leq 2 K \sqrt{\frac{2\log(K/\delta)}{t}} + \frac{8K\log(K/\delta)}{3t} \,.
	\]
	% \iffalse %%%%%%%%%%%%%%%%%%%%%%%
	If $t\leq  15 \beta K^2 / 9$, then,
	\[
	| q^*_t(i) - \bar q_t(i)|  \leq 1 \leq \frac{4 K}{3} \sqrt{\frac{\log(1/\delta)}{t}} \,.
	\]
	% \fi %%%%%%%%%%%%%%%%%%%%%%%
	Therefore, for all $\delta \in (0,1)$, with probability at least $1-\delta$, for all $1\leq i\leq K$
	\[
	| q^*_t(i) - \bar q_t(i)|  \leq 2 K \sqrt{\frac{2\log(K/\delta)}{t}} + \frac{8K\log(K/\delta)}{3t} \,.
	\]

\end{proof}

\subsection{Proof of Thm.~\ref{thm:reg_icat}}
\label{app:reg_icat}

\ubicat*

\begin{proof}
Consider any fixed set $S \subseteq [K]$, and suppose we run EXP3 algorithm on the set $S$, over any nonnegative sequence of losses $\hat \ell_1, \hat \ell_2, \ldots \hat \ell_T$ over items of set $S$, and consequently with weight updates $\q_1^S, \q_2^S, \ldots \q_T^S$ where as per EXP3 algorithm 
\begin{equation}
	\label{eq:qtdef}
	q_t^S(i) = \frac{e^{-\eta \sum_{\tau=1}^{t-1}\hat \ell_\tau(i)}}{\sum_{j \in S} e^{-\eta \sum_{\tau=1}^{t-1}\hat \ell_\tau(j)}} \,, \quad i \in S
\end{equation}
$\eta > 0$ being the learning rate of the EXP3 algorithm. Note here that for the analysis, we assume for this hypothetical EXP3 algorithm which plays on actions in $S$ only, the available sets are fixed to $S$ for all $t\in [T]$. We also consider that $q_t^S(i) = 0$ for $i \notin S$. 

Then from the standard regret analysis of the EXP3 algorithm it is known that \cite{PLG06} for all $i \in S$
\begin{align*}
 \sum_{t = 1}^{T}\big< \q_t^S,\hat \ell_t \big> - \sum_{t = 1}^T\hat \ell_t(i)  \le \frac{\log K}{\eta} + \eta \sum_{t = 1}^T \sum_{k \in S} q_t^S(k)\hat \ell_t(k)^2 \,.
\end{align*}
Let $\pi^*:S \mapsto [K]$ be any strategy. Then, applying the above regret bound to the choice $i = \pi^*(S)$ and taking the expectation over $S \sim P_{\a}$ and over the possible randomness of the estimated losses, we get
\begin{equation}
  \sum_{t = 1}^{T} \E\Big[\big< \q_t^S,\hat \ell_t \big>\Big] - \sum_{t = 1}^T\E\Big[\hat \ell_t(\pi^*(S))\Big] \le \frac{\log K}{\eta} + \eta \sum_{t = 1}^T \E\bigg[\sum_{k \in S} q_t^S(k) \hat \ell_t(k)^2 \bigg] \,.
 	\label{eq:RegretEXP3}
\end{equation}
Note that we did not make any assumptions on the estimated losses yet expect non-negativity.  

For simplicity we abbreviate $S \sim \P_{\a}$ as $S \sim \a$ henceforth. We denote the sigma algebra generated by the history of outcomes till time $t$ (i.e. $\{i_\tau,S_\tau\}_{\tau=1}^{t}$) by $\cH_{t}$. Then recall from Eqn. \eqref{eq:reg} that we wish to analyse the regret $R_T$, which is defined with respect to the loss sequence $\ell_1, \ell_2,\ldots  \ell_T$ and with activation sets $S_1,\dots,S_T$. That is, we need to upper-bound 
\[
R_{T} = \max_{\pi: 2^{[K]}  \mapsto [K]} \left\{ \sum_{t=1}^{T}\E\big[\ell_t(i_t)\big] - \sum_{t=1}^{T} \E\big[\ell_t(\pi(S_t)) \big] \right\} \,.
\]

Towards proving the above regret bound of \algicat\, from Inequality~\eqref{eq:RegretEXP3}, we now first establish the following lemmas that relates the different expectations of Inequality~\eqref{eq:RegretEXP3} with quantities related to the regret. 

\firstterm*

\begin{proof}
Let $t\in [T]$. We first consider the probabilistic event (denoted $\cE_{t}$) that~\eqref{eq:barq_conc} is true. That is, for all $i \in [K]$
\begin{equation}
	\label{eq:concentrationqbar}
	|q^*_t(i) - \bar q_t(i)| \le \lambda_t \,.
\end{equation}
Remark that $\cE_t$ is $\cH_{t-1}$ measurable since $\bar \q_t$ and $\q_t^*$ are $\cH_{t-1}$ measurable.

We start from the right hand side noting that:
\begin{align*}
\E &\Big[\big< \q_t^S,\hat \ell_t \big> \big|  \cE_t \Big]  = \E\bigg[\sum_{i \in S} q_t^S(i)\hat \ell_t(i) \Big|  \cE_t \bigg]\\
& = \E\Bigg[\E\bigg[\sum_{i \in S} q_t^S(i)\dfrac{\ell_{t}(i)\1(i_t = i)}{\bar q_t(i) + \lambda_t} \ \Big|\  S,S_t,\cH_{t-1}\bigg]  \bigg|\  \cE_t \Bigg] \hspace*{2cm} \leftarrow \text{by definition \eqref{eq:loss_estimate} of $\hat \ell_t$} \\
& = \E\Bigg[ \E\bigg[ \sum_{i \in S_t} q_t^S(i)\dfrac{\ell_{t}(i)q_t^{S_t}(i)}{\bar q_t(i) + \lambda_t} \bigg| S_t, \cH_{t-1} \bigg] \bigg| \ \cE_t\Bigg] \hspace*{2.9cm} \leftarrow \text{taking the expectation over $i_t \sim \q_t^{S_t}$} \\
& = \E\Bigg[\E\bigg[ \sum_{i \in S_t} q_t^*(i)\dfrac{\ell_{t}(i)q_t^{S_t}(i)}{\bar q_t(i) + \lambda_t} \ \bigg|\  \cH_{t-1} \bigg] \bigg| \cE_t \Bigg] \hspace*{3.3cm} \leftarrow \text{taking the expectation over $S \sim \a$} \\
& \ge \E\Bigg[\E\bigg[ \sum_{i \in S_t} q_t^*(i)\dfrac{\ell_{t}(i)q_t^{S_t}(i)}{q_t^*(i) + 2\lambda_t} \bigg| \cH_{t-1} \bigg] \bigg| \cE_t \Bigg]  \hspace*{3.35cm} \leftarrow \text{from Inequality~\eqref{eq:concentrationqbar}}\\
& \ge \E\Bigg[\E\bigg[ \sum_{i \in S_t} \ell_{t}(i)q_t^{S_t}(i)\Big( 1 - \frac{2\lambda_t}{q_t^*(i)}\Big) \bigg| \cH_{t-1}\bigg] \bigg| \cE_t \Bigg] \hspace*{2.45cm} \leftarrow \text{since $(1+x)^{-1} \geq 1-x$ for $x \geq 0$}  \\
& \geq \E\Bigg[\E\bigg[ \sum_{i \in S_t} \ell_{t}(i)q_t^{S_t}(i) \bigg| \cH_{t-1}\bigg]\bigg| \cE_t \Bigg] - \E\Bigg[ 2\lambda_t \sum_{i=1}^K \frac{\E\big[q_t^{S_t}(i)\big|\cH_{t-1}\big]}{q_t^*(i)} \ \bigg| \ \cE_t\Bigg] \leftarrow \text{since $q^*_t$ is $\cH_{t-1}$ measurable}\\
& = \E\bigg[\E\Big[ \sum_{i \in S_t} \ell_{t}(i)q_t^{S_t}(i) \Big| \cH_{t-1} \Big]\Big| \ \cE_t \bigg] -2K\lambda_t \hspace*{1cm} \leftarrow \text{by definition of $q^*_t(i) = \E\big[q_t^{S_t}(i) \big| \cH_{t-1}\big]$}\\ 
& = \E\big[ \ell_{t}(i_t) \big| \cE_t \big]  -2K \lambda_t\,.
\end{align*}

It remains now to deal with the case when the event $\cE_t$ (i.e., \eqref{eq:concentrationqbar}) is not satisfied. By Lem.~\ref{lem:conc_barq_p}, this happens with probability smaller than $\delta$. Therefore because the estimated losses are smaller than $1/\lambda_t$, we have 
\begin{align*}
	\E\Big[\big< \q_t^S,\hat \ell_t \big> \Big] & = \E\Big[\big< \q_t^S,\hat \ell_t \big>\ \big|\ \cE_t \Big] Pr(\cE_t) + \E\Big[\big< \q_t^S,\hat \ell_t \big>\big| \ \text{not } \cE_t  \Big] Pr\big(\text{not } \cE_t \big)\\
	& \geq  \E\big[\ell_t(i_t) \big| \cE_t] Pr(\cE_t) - 2K\lambda_t - \frac{Pr\big(\text{not } \cE \big)}{\lambda_t} \\
	& \geq \E\big[\ell_t(i_t)] - 2K\lambda_t - \frac{\delta}{\lambda_t} \,.
\end{align*}
\end{proof}

\secondterm*

\begin{proof} The proof follows a similar analysis to the one of Lem.~\ref{lem:first_term}. We start by assuming the probabilistic event $\cE_t$ of Inequality~\eqref{eq:concentrationqbar} and to ease the notation, we denote by $\E_t = \E\big[ \ \cdot \ \big| \cE_t\big]$ the conditional expectation given $\cE_t$ holds true. Then, 
\begin{align*}
\E_t\Big[ \hat \ell_t(i)  \Big]  & = \E_t\Bigg[\E \bigg[\dfrac{\ell_{t}(i)\1(i_t = i)}{\bar q_t(i) + \lambda_t}  \bigg| S_t, \cH_{t-1} \bigg] \Bigg] \hspace*{1.17cm} \leftarrow \text{by definition of $\hat \ell_t$}\\
& =  \E_t\Bigg[\E\bigg[\dfrac{\ell_{t}(i)q_t^{S_t}(i)}{\bar q_t(i) + \lambda_t} \bigg|  \cH_{t-1} \bigg] \Bigg]
\hspace*{2.1cm} \leftarrow \text{taking the expectation over $i_t \sim \q_t^{S_t}$}\\
& = \E_t\Bigg[\dfrac{\ell_{t}(i)q_t^*(i)}{\bar q_t(i) + \lambda_t}\Bigg] \hspace*{3.74cm} \leftarrow \text{taking the expectation over $S_t \sim \a$} \\
& \le  \E_t\Big[\ell_{t}(i)\Big] \hspace*{4.7cm} \leftarrow \text{by~\eqref{eq:concentrationqbar}} \\
& = \ell_t(i) \,.
\end{align*}
Similarly to the proof of Lem.~\ref{lem:first_term}, using that $\cE$ holds with probability at least $1-\delta$ and using that the estimated losses are in $[0,1/\lambda_t]$, we get
\[
	\E\Big[ \hat \ell_t(i)  \Big] \leq \ell_t(i) + \frac{\delta}{\lambda_t} \,.
\]
\end{proof}

\varterm*

\begin{proof}
Let $t \in [T]$. Then, denoting $\cE_t$ the event such that~\eqref{eq:concentrationqbar} holds, we can derive:
\begin{align*}
\E_t\Bigg[ \sum_{i \in S} q_t^S(i)\hat  \ell_t(i)^2\Bigg]  
& = \E_t\Bigg[\E\bigg[\sum_{i \in S} q_t^S(i)\dfrac{\ell_{t}^2(i)\1(i_t = i)}{(\bar q_t(i) + \lambda_t)^2} \bigg| S,S_t, \cH_{t-1} \bigg]\Bigg] \hspace*{.62cm} \leftarrow \text{by definition of $\hat \ell_t$}\\
& = \E_t\Bigg[\E\bigg[ \sum_{i \in S_t} q_t^S(i)\dfrac{\ell_{t}^2(i)q_t^{S_t}(i)}{(\bar q_t(i) + \lambda_t)^2} \bigg| S_t, \cH_{t-1}\bigg] \Bigg] \hspace*{1.1cm} \leftarrow \text{taking the expectation over $i_t \sim \q_t^{S_t}$}\\
& \le \E\Bigg[\E\bigg[ \sum_{i \in S_t} q_t^*(i)\dfrac{q_t^{S_t}}{(\bar q_t(i) + \lambda_t)^2} \bigg| \cH_{t-1}\bigg]\Bigg] \,,
\end{align*}
where, in the last inequality, we took the expectation over $S_t \sim \a$ and used that $\ell_t(i)^2 \leq 1$. Therefore, taking the expectation over $S_t \sim \a$, we get
\begin{align*}
\E_t\Bigg[ \sum_{i \in S} q_t^S(i)\hat  \ell_t(i)^2\Bigg]
& \leq \E_t\Bigg[\sum_{i \in [K]} \dfrac{q_t^*(i)^2}{(\bar q_t(i) + \lambda_t)^2} \Bigg] \leq K \,,
\end{align*}
where the last inequality is because under $\cE_t$, $|\bar q_t(i) - q_t^*(i)| \leq \lambda_t$. Now, using that $\hat \ell_t(i)^2 \leq 1/\lambda_t^2$, and using that $\cE_t$ is satisfied with probability at least $1-\delta$, we conclude the proof of the lemma:
\[
	\E\Bigg[ \sum_{i \in S} q_t^S(i)\hat  \ell_t(i)^2\Bigg] \leq K + \frac{\delta}{\lambda_t^2} \,.
\]
\end{proof}

Given the above claims in place, we are now in a position to prove the main theorem as shown below. Recall from Eqn. \eqref{eq:reg}, the actual regret definition of our proposed algorithm: 
\[
R_{T}(\text{\algicat}) = \max_{\pi: 2^{[K]} \mapsto [K]}  \sum_{t=1}^{T}\E\Big[\ell_t(i_t) - \ell_t(\pi(S_t))\Big]
\]
Denoting the best policy $\pi^* := \arg\min_{\pi: 2^{[K]}\mapsto [K]}\sum_{t = 1}^{T}\E_{S_t \sim P_{\a}}[\ell(\pi(S_t))]$, and combining the claims from Lem.~\ref{lem:first_term},~\ref{lem:second_term}, we  get:
\begin{align*}
R_{T}&(\text{\algicat}) =  \sum_{t=1}^{T}\E\Big[\ell_t(i_t) - \ell_t(\pi^*(S_t))\Big]\\
& \leq  \sum_{t=1}^{T}\E\Big[\big< \q_t^S,\hat \ell_t \big> + 2K \lambda_t + \frac{\delta}{\lambda_t} - \hat \ell_t(\pi^*(S_t)) + \frac{\delta}{\lambda_t} \Big]  \hspace*{1cm} \leftarrow \text{from Lemmas~\ref{lem:first_term} and~\ref{lem:second_term}}\\
& \le  2K \sum_{t = 1}^{T} \lambda_t + 2 \sum_{t=1}^T \frac{\delta}{\lambda_t} + \sum_{t=1}^{T}\E\Big[\big< \q_t^S,\hat \ell_t \big>  -  \hat \ell_t(\pi^*(S))\Big] \,.
\end{align*}
Then, we can further upper-bound the last term in the right-hand-side using Inequality~\eqref{eq:RegretEXP3} and  Lem.~\ref{lem:var_term}, which yields
\begin{align}
	R_{T}(\text{\algicat}) & \leq  2K \sum_{t = 1}^{T} \lambda_t + 2 \sum_{t=1}^T \frac{\delta}{\lambda_t}  + \frac{\log K}{\eta} + \eta \sum_{t = 1}^T \E\bigg[\sum_{k \in S} q_t^S(k)\hat \ell_t(k)^2 \bigg] \nonumber \\
		& \leq 2K \sum_{t = 1}^{T} \lambda_t + 2 \sum_{t=1}^T \frac{\delta}{\lambda_t}  + \frac{\log K}{\eta} + \eta KT + \eta  \sum_{t=1}^T \frac{\delta}{\lambda_t^2} \nonumber \\
		& \leq \frac{\log K}{\eta} + \eta KT  +  2K \sum_{t = 1}^{T} \lambda_t + 3 \sum_{t=1}^T \frac{\delta}{\lambda_t^2} \,,
		\label{eq:sleeping_regret_bound}
\end{align}
where in the last-inequality we used that  $\eta \leq 1$ and $\lambda_t \leq 1$. Otherwise, we can always choose $\min\{1,\lambda_t\}$ instead of $\lambda_t$ in the algorithm and Lem.~\ref{lem:conc_barq_p} would still be satisfied.

The proof is concluded by replacing $\lambda_t  = 2 K \sqrt{\frac{2\log(K/\delta)}{t}} + \frac{8K\log(K/\delta)}{3t}$ and by upper-bounding the two sums
\begin{align*}
	\sum_{t=1}^T \lambda_t  
		& = 2K \sqrt{2 \log \Big(\frac{K}{\delta}\Big)} \sum_{t=1}^T \frac{1}{\sqrt{t}} + \frac{8K}{3} \log \Big(\frac{K}{\delta}\Big) \sum_{t=1}^T \frac{1}{t} \\
		& \leq 2K \sqrt{2 \log \Big(\frac{K}{\delta}\Big) T} + \frac{8K}{3} \log \Big(\frac{K}{\delta}\Big) (1+\log T) 
\end{align*}
and using $\lambda_t \geq 2K\sqrt{2 \log(K/\delta)/t}$, we have
\begin{align*}
	\sum_{t=1}^T \frac{1}{\lambda_t^2}
		& \leq \frac{1}{8 K^2  \log (K/\delta)} \sum_{t=1}^T t \leq \frac{T^2}{8 K^2  \log (K/\delta)} \leq \frac{T^2}{8 K^2} \,.
\end{align*}
Then, using $\delta := K/T^2$, $\log(K/\delta) = 2 \log(T)$, we can further upper-bound:
\[
	\sum_{t=1}^T \lambda_t \leq 4K \sqrt{T\log T} + \frac{8K}{3} (1+\log T)(\log T) \leq 7K \sqrt{T\log T} 
\]
and
\[
	 3 \sum_{t=1}^T \frac{\delta}{\lambda_t^2} \leq \frac{3}{8K} \leq 1 \,. 
\]
Thus, upper-bounding the two sums into~\eqref{eq:sleeping_regret_bound}, we get
\[
	R_{T}(\text{\algicat}) \leq \frac{\log K}{\eta} + \eta KT + 14 K^2 \sqrt{T \log T} + 1 \,.
\]
Optimizing $\eta = \sqrt{(\log K)/KT}$ and upper-bounding $\sqrt{KT \log K} \leq K^2 \sqrt{T}$, we finally conclude the proof
\[
	R_{T}(\text{\algicat}) \leq 16 K^2 \sqrt{T \log T} + 1 \,.
\]
\end{proof} % proof of main Theorem 2

\subsection{Proof of Lem.~\ref{lem:conc_barq_p_e}}
\label{app:conc_icate}

\concbarqe*

\begin{proof}
Let $t \in [T]$. We start by remarking that for any $i \in [K]$:
\begin{align}
\label{eq:barq_conce0}
|q^*_t(i) - \tilde q_t(i)| \le |q^*_t(i) - \bar q_t(i)| + |\bar q_t(i) - \tilde q_t(i)| \,.
\end{align}

Note that $\q_t^{S^{(1)}_t}(i), \q_t^{S^{(2)}_t}(i), \ldots \q_t^{S^{(T)}_t}(i)$ are independent of each other given the past $\cH_{t-1}$. Furthermore, they are in $[0,1]$ and are unbiased estimates of $\bar q_t(i) = \E_{S \sim \hat \a_t}[q_t^S(i)]$, i.e., for all $i \in [K], \tau \in [t]$
\[
	\E_{\hat \a_t}\Big[\q_t^{S^{(\tau)}_t}(i)\Big] = \bar \q_t(i) \,.
\] 
Thus, using Hoeffding's inequality and a union bound over all $i \in [K]$, we get that with probability at least $(1-\delta/2)$:

\begin{align}
\label{eq:barq_conce1}
|\bar q_t(i) - \tilde q_t(i)| \le \sqrt{ \frac{1}{2t} \ln \frac{2K}{\delta}} \,.
\end{align}

Furthermore, from Lem.~\ref{lem:conc_barq_p} we have, for all $i \in [K]$, with probability at least $(1-\delta/2)$: 
\begin{align}
\label{eq:barq_conce2}
|q^*_t(i) - \bar q_t(i)| \le 2 K \sqrt{\frac{2\log(2K/\delta)}{t}} + \frac{8K\log(2K/\delta)}{3t}. 
\end{align}
The proof follows combining Eqn. \eqref{eq:barq_conce0}, \eqref{eq:barq_conce1} and \eqref{eq:barq_conce2} and using $K\geq 1$.
\end{proof}

\subsection{Proof of Thm.~\ref{thm:reg_icat_e}}
\label{app:thm_icate}

\ubicate*

\begin{proof}
The regret bound can be proved using the same proof technique used for Thm.~\ref{thm:reg_icat}, except replacing the concentration result of Lem.~\ref{lem:conc_barq_p_e} in place of Lem.~\ref{lem:conc_barq_p}.

Concerning the computational time. At each round $t\geq 1$, the Alg.~\ref{alg:icat} performs the following operations:
\begin{itemize}[label={--},nosep]
	\item[a)] update $\hat a_{ti}$ for all $i \in [K]$ \hspace*{10.5cm} $\rightarrow$ Cost $O(K)$
	\item[b)] for each $\tau \in [t]$, sample $S_t^{(\tau)}$ from $P_{\hat \a_t}$ and compute $\q_t^{S_t^{(\tau)}}(i) \propto p_t(i) \1\big(i\in S_t^{(\tau)}\big)$ for each $i\in [K]$ \hspace*{.2cm} $\rightarrow$ Cost $O(tK)$
	\item[c)] compute the estimated losses $\hat \ell_t(i)$ and update $p_{t+1}(i)$ for each $i\in [K]$ \hspace*{4.1cm} $\rightarrow$ Cost $O(K)$ 
\end{itemize}
The time complexity to perform iteration $t$ is therefore $O(tK)$. 

As for spatial complexity, the algorithm only needs to keep track of $\hat \a_t \in [0,1]^K$ and $\p_t \in [0,1]^K$. Since the step b) above can also be performed sequentially, the total storage complexity is $O(K)$.
\end{proof}

%%%%%%%%%%%%%%%%%%%%%%%%%%%%%%%%%%%%

\iffalse

\subsection{Proof of Thm.~\ref{thm:reg_icat_p}}
\label{app:reg_icatp}

\ubicatp*

\begin{proof}
	The result can be proved using the same proof technique used for Thm.~\ref{thm:reg_icat}, except replacing the concentration result of Lem.~\ref{lem:conc_barq_p} in place of Lem.~\ref{lem:conc_barq}.
\end{proof}

\fi

%!TEX root = sleeping_icml20.tex

\section{Appendix for Sec.~\ref{sec:algo_dep}}

\concbarqg*

\begin{proof}
Let $t \in [T]$. We start by first noting the concentration of $\hat P(S)$ for any $S \subseteq [K]$, which using Bernstein's inequality we know that given any fixed $\delta \in [0,1]$:
	
\begin{align*}
Pr\bigg( |P(S) - \hat P_t(S)| > \sqrt{\frac{2 P(S)(1-P(S))\ln \frac{1}{\delta}}{t}} + \frac{2\ln \frac{1}{\delta}}{3t} \bigg) \le \delta \,.
\end{align*}
				
Then consider the event 
\begin{equation}
	\label{eq:event_general}
	\cE_t := \left\{\forall S \subseteq [K]: |P(S) - \hat P_t(S)| >   \sqrt{\frac{2 P(S)(1-P(S))\ln \frac{2^K}{\delta}}{t}} + \frac{2\ln \frac{2^K}{\delta}}{3t}  \right\} \,.
\end{equation}
Taking a union bound over all $2^K$ subsets $S \subseteq [K]$, we get that: $Pr( \cE_t ) \ge 1- \delta$. 
Now recall in this case that by definition 
\begin{equation}
	q^*_t(i) := \E_{S\sim\P}[q_t^S(i)] = \sum_{S \in 2^{[K]}}P(S)q_t^S(i),
	\label{eq:qstar_sum}
\end{equation}
and similarly, 
\begin{equation}
\bar q_t(i) := \frac{1}{t} \sum_{\tau = 1}^{t} q_t^{S_\tau}(i) = \E_{S\sim\hat P_t(S)}[q_t^S(i)] = \sum_{S \in 2^{[K]}}\hat P_t(S)q_t^S(i),
\label{eq:qbar_sum}
\end{equation}
with $\hat P_t(S): = \frac{1}{t}\sum_{\tau = 1}^t\1(S_\tau = S)$. 
To ease the notation, let us denote $S_i = \1(i \in S)$. 

We now proceed to bound $|q_t^*(i) - \bar q_t(i)|$ for any item $i \in [K],\, t \in [T]$. Let us denote $Q_t(i) = \sum_{S \in 2^{[K]}}q_t^S(i)$. Note that $Q_t(i) \le 2^K$. Let us also define for simplicity $\alpha(S) = 2 P(S)(1-P(S))\ln (2^K/\delta)$ and $\beta = \frac{2}{3} \ln (2^K/\delta)$ the terms in the right-hand-side of~\eqref{eq:event_general}. 
Then, following the above claims we note that from \eqref{eq:event_general} with probability at least $(1-\delta)$,
	
\begin{align*}
%\label{eqn:conclem_2}
| q^*_t(i) - \bar q_t(i)| & \leq \sum_{S \in 2^{[K]}} q_t^S(i) \big|P(S) - \hat P_t(S)\big|  \hspace*{2cm} \leftarrow \text{from \eqref{eq:qstar_sum} and~\eqref{eq:qbar_sum}} \\
&  \le  \sum_{S \in 2^{[K]}}  q_t^S(i) \bigg(\sqrt{\frac{\alpha(S)}{t}} + \frac{\beta}{t}\bigg) \hspace*{1.9cm} \leftarrow \text{from \eqref{eq:event_general}} \\
& \le \frac{\beta 2^K}{t} +  Q_t(i) \sum_{S \in 2^{[K]}}  \frac{q_t^S(i)}{Q_t(i)}\sqrt{\frac{\alpha(S)}{t}}  \hspace*{1cm} \leftarrow \text{because $Q_t(i) \leq 2^K$} \\
& \leq \frac{\beta 2^K}{t} +  Q_t(i)  \sqrt{\sum_{S \in 2^{[K]}}  \frac{q_t^S(i)\alpha(S)}{Q_t(i)t}}  \hspace*{1.2cm} \leftarrow \text{from Jensen's inequality} \\
& \le \sqrt{\frac{2^{K+1}}{t} \ln \frac{2^K}{\delta}}+ \frac{2^{K+1}}{3t} \ln \frac{2^K}{\delta},
\end{align*}
where the last inequality is because $Q_t(i) \leq 2^K$ and
\[
	\sum_{S \in 2^{[K]}} q_t^{S}(i) \alpha(S) \leq \sum_{S \in 2^{[K]}} \alpha(S) \leq 2 \log(2^K/\delta) \sum_{S \in 2^{[K]}} P(S) \leq 2 \log(2^K/\delta) \,.
\]
This concludes the proof of the Lemma.
\end{proof}

\subsection{Proof of Thm.~\ref{thm:reg_icat_g}}
\label{app:reg_icat_g}

\ubicatg*

\begin{proof}%\textbf{(sketch)}.
The proof is almost similar to the proof of Thm.~\ref{thm:reg_icat}. We use the same notation introduced in that proof for ease of understanding. Same as before the proof relies on the Lemmas~\ref{lem:first_term},~\ref{lem:second_term}, and~\ref{lem:var_term} that are satisfied but for 
\[
	\lambda_t = \sqrt{\frac{2^{K+1}}{t} \ln \frac{2^K}{\delta}}+ \frac{2^{K+1}}{3t} \ln \frac{2^K}{\delta} \,,
\]
since we need to use the concentration Lem.~\ref{lem:conc_barq_g} instead of Lem.~\ref{lem:conc_barq_p}. Now, recall from Eqn. \eqref{eq:reg}, the actual regret definition of our proposed algorithm: 
\[
R_{T}(\text{\algicatg}) = \max_{\pi: 2^{[K]} \mapsto [K]}  \sum_{t=1}^{T}\E\Big[\ell_t(i_t) - \ell_t(\pi(S_t))\Big]
\]
Denoting the best policy $\pi^* := \arg\min_{\pi: 2^{[K]}\mapsto [K]}\sum_{t = 1}^{T}\E_{S_t \sim P_{\a}}[\ell(\pi(S_t))]$, and combining the claims from Lem.~\ref{lem:first_term},~\ref{lem:second_term} and~\ref{lem:var_term}, we get following the same arguments as the ones of Thm.~\ref{thm:reg_icat}:

\begin{align*}
R_{T}&(\text{\algicatg}) =  \sum_{t=1}^{T}\E\Big[\ell_t(i_t) - \ell_t(\pi^*(S_t))\Big]\\
& \le \sum_{t=1}^{T}\E\big[\big< \q_t^S,\hat \ell_t \big> + 2K \lambda_t + \frac{2\delta}{\lambda_t} \big] - \E[\hat \ell_t(\pi^*(S_t))] \hspace*{2.4cm} \leftarrow \text{from Lem.~\ref{lem:first_term} and Lem.~\ref{lem:second_term}} \\
& =  \sum_{t=1}^{T}\E\Big[\big< \q_t^S,\hat \ell_t \big> - \hat \ell_t(\pi^*(S)) \Big] +  2\sum_{t = 1}^{T} \big(K\lambda_t + \frac{\delta}{\lambda_t}\big) \\ 
& \leq  \frac{\log K}{\eta} +  \eta \sum_{t = 1}^T \E\bigg[\sum_{k \in S} q_t^S(k)\hat \ell_t(k)^2 \bigg] +  2\sum_{t = 1}^{T} \big(K\lambda_t + \frac{\delta}{\lambda_t}\big)  \hspace*{1cm} \leftarrow \text{from Inequality~\eqref{eq:RegretEXP3}} \\
& \leq \frac{\log K}{\eta} +  \eta KT +  \sum_{t = 1}^{T} \big(2 K\lambda_t + \frac{2\delta}{\lambda_t} + \frac{\eta \delta}{\lambda_t^2} \big)  \hspace*{3.2cm} \leftarrow \text{from Lem.~\ref{lem:var_term}} \\
& \leq  \frac{\log K}{\eta} +  \eta KT +  \sum_{t = 1}^{T} \big(2 K\lambda_t + \frac{3 \delta}{\lambda_t^2} \big)  \hspace*{4.1cm} \leftarrow \text{because $\eta \leq 1$ and $\lambda_t \leq 1$}
\end{align*}
To conclude the proof, it only remains to compute the sums and to choose the parameters $\delta = 2^K / T^2$ and $\eta = \sqrt{(\log K)/KT}$. 
Using $\lambda_t \geq \sqrt{2^{K+1} /t}$, we have
\[
	\sum_{t=1}^T \frac{\delta}{\lambda_t^2} \leq  \frac{\delta T^2}{2^{K+1}} \leq 1
\]
and since $\log(2^K/\delta) = 2 \log T$, 
\[
	\lambda_t = \sqrt{\frac{2^{K+1}}{t} \ln T }+ \frac{2^{K+1}}{3t} \ln T
\]
which entails
\[
	\sum_{t=1}^T \lambda_t \leq \sqrt{2^{K+1} T \log T} + \frac{2^{K+1}}{3} (\log T)(1+ \log T) \leq \sqrt{2^{K+1} T \log T} + 2^{K+2} (\log T)^2\,.
\]
Substituting $\eta$ and these upper-bounds in the regret upper-bound concludes the proof:
\begin{align*}
	R_{T}(\text{\algicatg}) 
		& \leq 2 \sqrt{KT \log K} + K \sqrt{2^{K+3} T \log T} + K2^{K+3}(\log T)^2 \\
		& \leq K \sqrt{2^{K+4} T \log T} + K2^{K+3}(\log T)^2 \,.
\end{align*}

As for the complexity, the only difference with Alg.~\ref{alg:icat} comes from the computation of $\bar q_t(i)$. The latter can also be performed with a computational cost of $O(tK)$. Yet, the algorithm needs to keep in memory the empirical distribution of $S_1,\dots,S_t$. Thus, a space complexity of $O(K + \min\{tK,2^K\})$. 
\end{proof} % proof of main Theorem 4

}

%%%%%%%%%%%%%%%%%%%%%%%%%%%%%%%%%%%%%%%%%%%%%%%%%%%%%%%%%%%%%%%%%%%%%%%%%%%%%%%
%%%%%%%%%%%%%%%%%%%%%%%%%%%%%%%%%%%%%%%%%%%%%%%%%%%%%%%%%%%%%%%%%%%%%%%%%%%%%%%

\end{document}